\documentclass{article} 
\usepackage{times}

\usepackage{arxiv}

\usepackage{amsmath,amsfonts,bm}

\def\eqref#1{equation~\ref{#1}}

\def\1{\bm{1}}

\DeclareMathAlphabet{\mathsfit}{\encodingdefault}{\sfdefault}{m}{sl}
\SetMathAlphabet{\mathsfit}{bold}{\encodingdefault}{\sfdefault}{bx}{n}

\usepackage[utf8]{inputenc} 
\usepackage[T1]{fontenc}    
\usepackage{hyperref}       
\usepackage{url}            
\usepackage{booktabs}       
\usepackage{amsfonts}       
\usepackage{nicefrac}       
\usepackage{microtype}      
\usepackage[table]{xcolor}
\usepackage[ruled,linesnumbered]{algorithm2e}
\usepackage[export]{adjustbox}
\usepackage{caption}
\usepackage{subcaption}
\usepackage{multirow}
\usepackage{multicol}
\usepackage{graphicx}
\usepackage{enumerate}
\usepackage{makecell}
\usepackage{tikz}
\usepackage{dsfont}
\usepackage{mathtools}
\usepackage{amsmath, amssymb, amsthm}
\usepackage{enumitem}
\usepackage[most]{tcolorbox}
\usepackage{colortbl}
\usepackage{wrapfig}
\usepackage{float}

\newtheorem{theorem}{Theorem}[section]

\newtheorem{lemma}[theorem]{Lemma}

\newtheorem{definition}{Definition}[section]
\newcommand{\norm}[1]{\left\lVert#1\right\rVert}
\newcommand\numberthis{\addtocounter{equation}{1}\tag{\theequation}}

\definecolor{lavender}{RGB}{230, 230, 249} 
\definecolor{aliceblue}{RGB}{242, 248, 255} 
\definecolor{blu}{RGB}{153, 204, 255} 
\definecolor{pur}{RGB}{222, 226, 250} 

\tcbset{
  mybase/.style={
    coltext=black,
    fonttitle=\bfseries,
    boxrule=0.2pt,
    arc=0.7mm,
    left=0.3mm,
    right=0.3mm,
    top=0mm,
    bottom=0mm,
    width=0.88\linewidth,
  }
}

\newtcolorbox{optionbox1}[1][]{
  mybase,
  colframe=blu!75!black,
  colback=blu!20!white,
  before=\vspace{1pt},
  after=\vspace{-2pt},
  width=0.35\linewidth,
  #1
}

\newtcolorbox{optionbox2}[1][]{
  mybase,
  colframe=pur!75!black,
  colback=pur!50!white,
  before=\vspace{0pt},
  after=\vspace{-2pt},
  #1
}

\title{
Zeroth-Order Sharpness-Aware Learning with Exponential Tilting}

\author{%
Xuchen Gong \quad Tian Li \\
University of Chicago \\
\texttt{\{xuchengo,litian\}@uchicago.edu}
}

\iclrfinalcopy

\begin{document}

\maketitle
\begin{abstract}
    Classic zeroth-order optimization approaches typically optimize for a smoothed version of the original function, i.e., the expected objective under randomly perturbed model parameters. This can be interpreted as encouraging the loss values in the perturbation set to be small on average. Popular sharpness-aware minimization (SAM) objectives, however, typically focus on the largest loss within the neighborhood to arrive at flat minima more effectively. In this work, we connect zeroth-order optimization (and its corresponding objectives) with SAM approaches explicitly, through an exponential tilting objective that provides a smooth transition between the \texttt{average}- and the \texttt{max}-loss formulations. We explore new zeroth-order algorithms to solve a \textit{soft} SAM objective parameterized by a tilting parameter $t$. We provide precise characterizations of the sharpness notions of the tilted SAM framework. Practically, our approach can be used as a gradient-free and memory-efficient alternative to SAM variants, and it achieves better generalization compared to vanilla zeroth-order baselines on a wide range of downstream tasks, including classification, multiple choice QA, and language generation.
\end{abstract}
\section{Introduction} \label{sec:intro}

\thinmuskip=1mu
\medmuskip=1mu
\thickmuskip=1mu
Zeroth-order optimization has gained traction when the first-order or higher-order gradient access is unavailable, unreliable, or expensive. Applications include black-box adversarial  attacks~\citep{ZOO}, fine-tuning large language models~\citep{deepzero,mezo, ZObenchmark}, differentially private learning~\citep{dpzero, dpzero_saeed}, and science problems.
Consider the standard empirical risk minimization (ERM) problem: $f(x) \coloneqq \frac{1}{N} \sum_{i=1}^N f(x;\xi_i)$ where $x \in \mathbb{R}^d$ is the model parameters and $\{\xi_i\}_{i\in [N]}$ represents training samples. One of the most popular zeroth-order algorithms relies on two function evaluations in the opposite directions to estimate the gradients. Such a two-point estimator takes the updating rule  
$G(x, \rho, u) \coloneqq (1/2\rho)[f(x + \rho u) - f(x - \rho u)] u$,
where $u$ is a random direction sampled from some distribution $\mu(u)$ (e.g., uniform over a sphere or Gaussian), and $\rho > 0$ is a smoothing parameter. 

Under mild assumptions, prior literature has shown that the two-point estimator optimizes an approximated, smooth version of the original function, i.e., $\mathbb{E}_u[G(x, \rho, u)] = \nabla_x \mathbb{E}_{v}[f(x+\rho v)]$\footnote{The distribution of $v$ depends on that of $u$; see Section~\ref{sec:method} for details.}~\citep{flaxman2004online}. In other words, the zeroth-order method effectively minimizes the expected loss $\mathbb{E}_{v} [f(x+\rho v)]$ in some perturbed neighborhood around $x$. 
Under such interpretation, zeroth-order optimization has a critical benefit that it is not originally designed for---ensuring the loss is small on average within the neighborhood so that the local minima can be flatter~\citep{wen2022does, universal, zoflat}. 
For over-parameterized and non-convex models, encouraging flatter local minima (i.e., optimizing a sharpness-aware objective) can be an effective technique that improves generalization performance \citep{sam_paper,sam_lm}. 

However, the aforementioned vanilla zeroth-order estimate can only be viewed as a special sharpness-aware minimization (SAM) objective that focuses on the average loss. The canonical SAM approach and its variants typically uses a min-max formulation~\citep{sam_paper}, which have been extensively studied in prior works and demonstrated  strong empirical performance~\citep[e.g.,][]{wu2020adversarial,tram}. 

In this work, observing the connections between zeroth-order optimization and SAM, we explore the explicit bias of zeroth-order optimization towards flat solutions in detail. We develop new zeroth-order algorithms that solve a continuous spectrum of sharpness-aware objectives, ranging from the average- to the max-loss formulations, leveraging exponential tilting. Exponential tilting has been used as a common technique to create parametric distribution shifts in various contexts~\citep{laplace,tiltedloss,robey2022probabilistically}. It has also been used to develop new sharpness-aware objectives that reweigh different local minima~\citep{tsam}. Our zeroth-order algorithms solve a similar objective to arrive at flatter solutions, while preserving the same computational and memory efficiency as the classic zeroth-order methods.


To be more specific, we consider a soft SAM objective parameterized by a tilting parameter $t$ (named tilted SAM or $t$-SAM~\citep{tsam}) that covers the average and min-max formulation as special cases. To approximate the unbiased gradient estimator of the titled SAM objective, we propose different strategies based on finite function evaluations under random perturbations of the model parameter. Additionally, we provide the precise characterizations of a family of sharpness notions of the tilted SAM framework and the solutions it favors, as a function of the tilting parameter $t$. While our framework in principle applies to any form of model perturbations, we investigate the cases with Gaussian and ball-constrained uniform perturbations  in detail. 

Our \underline{z}eroth-order \underline{e}xponential-tilted \underline{s}harpness-aware \underline{t}raining (ZEST) approach achieves superior performance compared with vanilla two-point estimator (corresponding to solving an average-loss based sharpness-aware objective) across various model types and downstream tasks, including classification, multiple choice QA, and language generation (Section~\ref{sec:exp}). In applications where zeroth-order optimization is competitive in general, ZEST can even achieve higher accuracies than first-order SAM variants while being gradient-free and memory-efficient (Section~\ref{sec:opt}). 

In summary, our contributions are as follows.  In Section~\ref{sec:method}, we propose a new zeroth-order optimization algorithm (ZEST) that uses exponential tilting to recover a smooth spectrum of sharpness-aware objectives. \textbf{Theoretically}, we analyze the explicit bias (the ``sharpness'' notion) of the $t$-SAM objective and illustrate how ZEST can reach flatter minima than the baselines. We show that our method can identify and conservatively avoid minima with large curvatures in any direction, while vanilla zeroth-order methods cannot (Section~\ref{sec:sharpness}). \textbf{Empirically}, in Section~\ref{sec:exp}, we evaluate ZEST on comprehensive language tasks and different model types, demonstrating that ZEST performs better than zeroth-order baselines while being equally fast and memory-efficient. 




\section{Preliminaries and Related Work} \label{sec:related}


\paragraph{Zeroth-Order Optimization.} Zeroth-order methods \citep{spall2002multivariate,shamir2017optimal,bach2016highly,duchi2015optimal,jamieson2012query,liu2018zeroth,nesterov2017random,agarwal2011stochastic} have gained recent attention due to their promising performance in fine-tuning language models and their memory efficiency, at the cost of increased iteration complexity compared to first-order methods~\citep{mezo,dpzero}. Zeroth-order methods typically optimize for a smoothed version of the functions, which can be interpreted as the expected loss values under perturbed model parameters. Enforcing that the loss values are small in expectation has connections with a special case of sharpness-aware approaches~\citep{zoflat}, where sharpness is defined as the trace of the Hessian~\citep{wen2022does}. Specifically, by Taylor expansion, the effective objective (proved in Appendix~\ref{app:related}) is
\begin{align*}
    \mathbb{E}_v[f(x+\rho v)] = \underbrace{f(x)}_{\text{Empirical loss}} + \underbrace{\frac{\rho^2}{2}\text{Tr}(\nabla^2 f(x))}_{\text{Sharpness }R_{\text{avg}}} + O(\rho^2d).
\end{align*}
In this work, we develop new zeroth-order algorithms that solve a spectrum of SAM objectives that cover this special case (Section~\ref{sec:method}) and provide precise characterizations of sharpness in our approach (Section~\ref{sec:sharpness}).

\paragraph{Sharpness-Aware Minimization.} Sharpness-Aware Minimization (SAM) and its variants have been extensively studied in prior work~\citep{sam_paper,randomsam,asam,bartlett2023dynamics,mi2022make,sabo,esam,wen2022does,baek2024sam,universal,towards,long2024sharpness}. The popular SAM objective minimizes the worst-case loss over perturbed parameters so that the loss values are uniformly small near the local minimum~\citep{sam_paper}. The problem is defined as 
\begin{align}
    \min_x \max_{\norm{\epsilon}\leq \rho} f(x+\epsilon), \label{eq:sam}
\end{align}
 where $\rho$ is the radius of the ball around $x \in \mathbb{R}^d$. To fully realize the potential of zeroth-order approaches and due to the difficulty in optimizing for this objective without gradient access, we propose to leverage an exponentially-tilted objective that can smoothly approximate this min-max formulation. 
In particular, we consider the tilted sharpness-aware minimization ($t$-SAM) objective~\citep{tsam}, which is paramaterized by a hyperparameter $t>0$ as
\begin{align}
    F_t(x) = \frac{1}{t} \log \mathbb{E}_{\mu(\epsilon)}\left[e^{tf(x+\epsilon)}\right], \label{eq:tsam}
\end{align} 
where $\mu(\cdot)$ denotes the distribution density of the perturbation. For instance, $\mu(\epsilon)$ can be the uniform distribution over an $L_2$ ball with radius $\rho$, i.e., $\|\epsilon\| \leq \rho$. This objective has been demonstrated to have superior empirical performance compared with the vanilla SAM formulation Eq.~(\ref{eq:sam}) for $0 < t < \infty$. When $t\to 0$, we have $F_t(x)\to \mathbb{E}_{\norm{\epsilon}\leq \rho}[f(x+\epsilon)]$, and optimizing it effectively corresponds to running gradient descent using the vanilla zeroth-order gradient estimators. As $t\to\infty$, $F_t(x)\to \max_{\norm{\epsilon}\leq \rho} f(x+\epsilon)$. We note that although ZEST optimizes a family of sharpness-regularized $t$-SAM objectives, there exist other sharpness-aware objectives and sharpness definitions that we leave for future work~\citep{universal,sabo}. 

\section{Zeroth-Order Tilted Sharpness-Aware Learning} \label{sec:method}

In this section, we introduce our main zeroth-order algorithm for sharpness-aware learning. 
In Section~\ref{sec:gradient}, we first derive the a gradient estimate for the $t$-SAM objective that only relies on function evaluations. 
Next, in Section~\ref{sec:estimate}, we propose two ways to approximate the gradient estimate using a small finite number of model perturbations. 
Our complete algorithm is presented in Algorithm~\ref{alg:zest}.

\subsection{Tilted Zeroth-Order Gradient} \label{sec:gradient}

In this section, we formally present the zeroth-order gradient for the tilted objective. We note that the first-order gradient of the $t$-SAM objective is
$\nabla_x F_t(x) = \frac{\mathbb{E}_{\mu(\epsilon)}[e^{tf(x+\epsilon)}\nabla f(x+\epsilon)]}{\mathbb{E}_{\mu(\epsilon)}[e^{tf(x+\epsilon)}]}$. To obtain this with access to only function evaluations, our main step is to substitute the integration of gradients with the integration of function values. Therefore, we use the divergence theorem \citep{divergence} when the perturbation is sampled from a uniform ball and Stein's lemma \citep{stein} when the perturbation follows Gaussian. We have the following theorem to approximate $t$-SAM gradients.

\begin{theorem}[Tilted Zeroth-Order Gradient] \label{thm:ZEST}
Denote $\mathcal{N}\coloneqq \mathcal{N}(0, I_d)$, $\mathcal{S}\coloneqq \mathcal{U}(\sqrt{d}\mathbb{S}^{d-1})$, i.e., uniform distribution over the sphere $\{v\in \mathbb{R}^d: \norm{v}=\sqrt{d}\}$, and $\mathcal{B}\coloneqq \mathcal{U}(\sqrt{d}\mathbb{B}^d)$, i.e., uniform distribution over the ball $\{v\in \mathbb{R}^d: \norm{v}\leq\sqrt{d}\}$. 
Denote $\rho$ as a perturbation scale. Let $f(x)<\infty$ and $t\in(0,\infty)$ such that $\int_v e^{tf(x+\rho v)} dv$ is integrable for any $x$ in the optimization trajectory with $v$ sampled from $\mathcal{N}$ or $\mathcal{B}$. Then the $t$-SAM objective (\ref{eq:tsam}) has unbiased zeroth-order gradients. Specifically, 

(1) with $F_t(x)=\frac{1}{t} \log \mathbb{E}_{v\sim\mathcal{N}} [e^{tf(x+\rho v)}]$, we have
\begin{align}
    \nabla_x F_t(x) = \frac{1}{t\rho}\frac{\mathbb{E}_{v\sim\mathcal{N}}[(e^{tf(x+\rho v)} - e^{tf(x-\rho v)}) v]}{\mathbb{E}_{v\sim\mathcal{N}}[e^{tf(x+\rho v)} + e^{tf(x-\rho v)}]};\label{eq:nn}
\end{align}
(2) with $F_t(x)=\frac{1}{t} \log \mathbb{E}_{v\sim\mathcal{B}} [e^{tf(x+\rho v)}]$, we have
\begin{align}
    \nabla_x F_t(x) & = \frac{1}{t\rho}\frac{\mathbb{E}_{v\sim\mathcal{S}}[(e^{tf(x+\rho v)} - e^{tf(x-\rho v)}) v]}{\mathbb{E}_{v\sim\mathcal{B}}[e^{tf(x+\rho v)} + e^{tf(x-\rho v)}]} \label{eq:sb} \\
    & \approx \frac{1}{t\rho}\frac{\mathbb{E}_{v\sim\mathcal{S}}[(e^{tf(x+\rho v)} - e^{tf(x-\rho v)}) v]}{\mathbb{E}_{v\sim\mathcal{S}}[e^{tf(x+\rho v)} + e^{tf(x-\rho v)}]}.\label{eq:ss}
\end{align}
\end{theorem}
We present the proofs in Appendix~\ref{app:proof} and make two remarks here. \underline{First}, as $t\to 0$, $t$-SAM reduces to the average-loss SAM objective $\mathbb{E}[f(x+\epsilon)]$, and our tilted zeroth-order gradient also reduces to the vanilla zeroth-order gradient. As $t\to\infty$, $t$-SAM approaches the max-loss SAM objective (Eq.~(\ref{eq:sam})), while Theorem~\ref{thm:ZEST} approaches the regime where integrability is not defined---it is expected since max-loss SAM is not differentiable. \underline{Second}, from Eq.~(\ref{eq:sb})  to Eq. (\ref{eq:ss}), we change the denominator from taking expectation over the ball $\mathcal{B}$ to over the sphere $\mathcal{S}$. This reduces computational cost by using the same sampled perturbations on the sphere to compute both the numerator and the denominator (Section~\ref{sec:estimate}). Theoretically, the bias of this approximation is controlled by $O(1/\sqrt{d})$ as most of the volume of a high-dimensional ball is concentrated near its boundary (the sphere).

\subsection{Estimates of Ratio-of-Expectations} \label{sec:estimate}

Our proposed tilted zeroth-order gradients (Eq.~(\ref{eq:nn}$-$\ref{eq:ss})) compute the ratio of expectations w.r.t. the sampled perturbations, but we can only sample a finite number of perturbations in practice. Given $k$ sampled perturbations $\{v_i\}_{i\in [k]}$ along with their function evaluations $\{e^{tf(x+\rho v_i)}\}_{i \in [k]}$ and $\{e^{tf(x-\rho v_i)}\}_{i\in [k]}$, our goal is to estimate $\frac{\mathbb{E}_v \left[\left(e^{tf(x+\rho v)} - e^{tf(x-\rho v)}\right) v\right]}{\mathbb{E}_v \left[e^{tf(x+\rho v)} + e^{tf(x-\rho v)}\right]}$ (Eq.~(\ref{eq:nn}$-$\ref{eq:ss})). In order words, if we denote $A_i = (e^{tf(x+\rho v_i)} - e^{tf(x-\rho v_i)}) v_i$ and $B_i = e^{tf(x+\rho v_i)} + e^{tf(x-\rho v_i)}$ for $ i\in[k]$, then we would like to estimate $\frac{\mathbb{E}[A]}{\mathbb{E}[B]}$, which is a ratio-of-expectation estimation problem. 

There are multiple well-studied ratio estimates in  statistics \citep{Tin}. In this section, we derive two economic choices, discuss their bias, and present our ZEST algorithm that leverages finite-perturbation estimates. For notation brevity, we denote $a_i^+=e^{tf(x+\rho v_i)}$, $a_i^-=e^{tf(x-\rho v_i)}$, and $Z=\sum_{i\in[k]} a_i^+ + a_i^-$. We denote the normalized values as $\bar{a}_i^+ \coloneqq a_i^+/Z$ and $\bar{a}_i^- \coloneqq a_i^-/Z$.

\paragraph{Naive Plug-In.} A natural ratio estimate is $\frac{\bar{A}}{\bar{B}}$ where $\bar{A}$ and $\bar{B}$ are the sample means for the current iteration. Therefore, we sample $\{v_i\}_{i\in[k]}$ from the given perturbation distribution and compute the sample mean of the numerator and denominator, respectively, which gives us
\begin{align}
    G_{\text{N}}^k \coloneqq \frac{1}{t\rho} \frac{\sum_{i=1}^k  A_i}{\sum_{i=1}^k  B_i}  = \frac{1}{t\rho} \sum_{i=1}^k (\bar{a}_i^+ - \bar{a}_i^-) v_i. \label{eq:naive}
\end{align}
Note that due to $\mathbb{E}[\frac{\bar{A}}{\bar{B}}]\neq \frac{\mathbb{E}[A]}{\mathbb{E}[B]}$ by Jensen's inequality, the naive plug-in is only asymptotically unbiased. When $k<\infty$, its bias reduces at rate $O(1/k)$ \citep{Beale}.


\paragraph{Bias-Corrected Plug-In.} Due to the constraint of small $k$'s  in practice, we derive a bias-corrected estimator, following \citet{van2000mean}. The Taylor expansion of $\mathbb{E}[\frac{\bar{A}}{\bar{B}}]$ gives us $\mathbb{E}[\frac{\bar{A}}{\bar{B}}] \approx \frac{\mathbb{E}[\bar{A}]}{\mathbb{E}[\bar{B}]} + \text{bias}$. By using the estimate with the bias term subtracted, we have
\begin{align}
     G_{\text{BC}}^k \coloneqq \frac{1}{t\rho} \sum_{i=1}^k \left\{1+\frac{k}{k-1}[\bar{a}_i^+ + \bar{a}_i^- - \sum_{i=1}^k (\bar{a}_i^+ + \bar{a}_i^-)^2] \right\} (\bar{a}_i^+ - \bar{a}_i^-) v_i, \label{eq:BC}
\end{align}
and the complete derivation of the bias term is in Appendix~\ref{app:proof:BC}. $G_{\text{BC}}^k$ has an improved bias reduction rate $O(1/k^2)$ \citep{BC} and has the same memory/computational complexity as the vanilla zeroth-order gradient estimator, because the computation and storage cost of $k$ exponential loss values is negligible. 

With the above two options derived, we introduce our ZEST algorithm and present its memory-efficient implementation in Algorithm~\ref{alg:zest}. In each iteration, we first sample $k$ perturbations iteratively using random seeds and record the normalized tilted loss values (Line 3-7). For memory efficiency, the perturbations will be deleted once these loss values are computed. Next, we obtain the weight for each perturbation using the chosen ratio estimate (Line 8-9). Finally, we re-generate the perturbations via the same random seeds and update the model parameters (Line 10-13). Since we sample and recover the perturbations in place without storing them in memory, ZEST is more memory-efficient than the first-order optimizer for $t$-SAM \citep{tsam}. See a detailed memory analysis in Section~\ref{sec:exp}. 

\SetKwInOut{Input}{Input}
\SetKwFunction{func}{Perturb}
\SetKwComment{Comment}{// }{}

\setlength{\algomargin}{1.5em}
\begin{algorithm} 
\DontPrintSemicolon
\caption{ZEST}
\label{alg:zest}
\Input{ $x \in \mathbb{R}^d$, tilting parameter $t$, perturbation scale $\rho$, number of queries $k$, learning rate $\eta$}
\vspace{0.3em}
\For{\textup{each iteration}}{
    Sample a batch of training data $\mathcal{D}$ and seeds $\{s_i\}_{i\in[k]}$\;
    \vspace{0.2em}
    \For{$i = 1, \cdots, k$}{
        Sample $v_i \sim \mathcal{N}(0, I_d)$ or $\mathcal{U}(\sqrt{d}\mathbb{S}^{d-1})$ based on seed $s_i$\;
        Compute $a_i^+ \gets e^{tf(x + \rho v_i;\mathcal{D})}$, $a_i^-\gets e^{tf(x - \rho v_i;\mathcal{D})}$ \;
    }
    \vspace{0.2em}
    Compute $Z \gets \sum_{i=1}^k a_i^+ + a_i^-$ and $\bar{a}_i^+ \gets a_i^+/Z$, $\bar{a}_i^- \gets a_i^-/Z$ for $i\in[k]$ \;
    Compute $w_i$ for $i\in[k]$ by\;
    \begin{optionbox1}
    \textbf{Option 1 (Naive):} $w_i \gets \bar{a}_i^+ - \bar{a}_i^-$
    \end{optionbox1}
    \begin{optionbox2}
    \textbf{Option 2 (Bias-corrected):} $w_i\gets \left\{1+\frac{k}{k-1}[\bar{a}_i^+ + \bar{a}_i^- - \sum_{i=1}^k (\bar{a}_i^+ + \bar{a}_i^-)^2] \right\} (\bar{a}_i^+ - \bar{a}_i^-)$
    \end{optionbox2} \;
    \vspace{0.2em}
    \For{$i = 1, \cdots, k$}{
        Recover $v_i \sim \mathcal{N}(0, I_d)$ or $\mathcal{U}(\sqrt{d}\mathbb{S}^{d-1})$ based on seed $s_i$\;
        $x \gets x - \eta (w_i/t\rho) * v_i$
    }
}
\end{algorithm}
\section{Sharpness Notions} \label{sec:sharpness}

In this section, we analyze the explicit bias (i.e., sharpness notions) of the $t$-SAM objective under both Gaussian (Section~\ref{sec:Rt:N}) and uniform ball perturbation (Section~\ref{sec:Rt:B}). 
Recall that updating via the vanilla zeroth-order gradient estimator is essentially minimizing $\mathbb{E}_{v}[f(x+\rho v)]$, which can be decomposed to the empirical loss $f(x)$ term plus a sharpness regularization term $R_{\text{avg}}\propto \text{Tr}(\nabla^2 f(x))$ (Section~\ref{sec:related}). 
We decompose the $t$-SAM objective into 
\begin{align*}
    F_t(x) = f(x) + R_t(x) + O(\rho^2d)
\end{align*}
where $f(x)$ is the empirical loss, $R_t(x)$ is the regularizer (used as our sharpness notion) dependent on $t$, and $O(\rho^2d)$ is the Taylor expansion error that can be controlled by taking proper $\rho$'s. Across two perturbation distributions, we show that as $t\to 0$, $R_t$ reduces to $R_{\text{avg}}$; as $t$ increases, $R_t$ increasingly relies on the gradient component in the top eigenspace of the Hessian $\nabla^2 f(x)$ and its top eigenvalues; as $t\to\infty$ (when admissible), $R_t$ exclusively relies on the gradient component projected to the first Hessian eigenvector and the largest eigenvalue. Therefore, our regularizer $R_t$ represents a spectrum of sharpness notions that promote ``flatter'' solutions. In Section~\ref{sec:toy}, we present a low-dimensional toy problem to illustrate (1) the different convergence behaviors of ZEST in contrast to vanilla zeroth-order methods due to different sharpness notions and (2) when and how our notion is superior. 

In the following, we start by defining \textit{sharpness sensitivity}, a notion that describes how the value of an eigenvalue impacts the sharpness regularizer $R_t$.

\begin{definition}[Sharpness Sensitivity]
The dependence of sharpness $R_t$ on $x$ can be re-expressed as its dependence on the Hessian eigenvalues $\{\lambda_i\}_{i=1}^d$ and the components of $\nabla f(x)$ in the eigenspace. We define the sharpness sensitivity to an arbitrary $\lambda_i$ as 
\begin{align}
    \phi_i(t) \coloneqq \frac{\partial R_t}{\partial \lambda_i}, \label{eq:phi_i}
\end{align}
which indicates how much impact the value of an arbitrary $\lambda_i$ has on the value of $R_t$.
\end{definition}
We note that if $\phi_i$ increases as $\lambda_i$ increases, $R_t$ is more dominated by large eigenvalues. Alternatively, if $\phi_i$ remains the same regardless of the value of $\lambda_i$, $R_t$ penalizes each eigenvalue equally and thus favors solutions with  small \textit{average} eigenvalues. With this quantity, we analyze $R_t$ when the perturbation is sampled from $\mathcal{N}(0,I_d)$ (denoted as $\mathcal{N}$) and $\mathcal{U}(\sqrt{d}\mathbb{B}^d)$ (denoted as $\mathcal{B}$) as follows.

\subsection{Gaussian Perturbation} \label{sec:Rt:N}

We derive the regularizer under Gaussian perturbation in this section. The Taylor expansion of $f(x+\rho v)$ with $v\sim\mathcal{N}$ is
\begin{align*}
     f(x+\rho v) & = f(x) + \rho \nabla f(x)^{\top}v + \frac{\rho^2}{2} v^{\top} \nabla^2f(x)v + O(\rho^2 \norm{v}^2),
\end{align*}
where $O(\rho^2\norm{v}^2)=O(\rho^2d)$ with high probability for large $d$. Therefore, with high probability, the Taylor expansion of the $t$-SAM objective is
\begin{align}
F_t(x) = f(x) + \underbrace{\frac{1}{t} \log \mathbb{E}_{v\sim\mathcal{N}}\left[\exp\left(t[\rho \nabla f(x)^{\top}v + \frac{\rho^2}{2} v^{\top} \nabla^2f(x)v ]\right)\right]}_{\text{ Sharpness } R_t} + O(\rho^2d).
\end{align}
We decompose Hessian $\nabla^2 f(x)$ into $\nabla^2 f(x) = Q^{\top}\Lambda Q$, where $Q$ is orthogonal with columns $\{e_1, \ldots, e_d\}$ that are ordered Hessian eigenvectors, and $\Lambda = \text{diag}(\lambda_1, \lambda_2, \ldots, \lambda_d)$ where $\lambda_1\geq \ldots \geq \lambda_d$ are the ordered Hessian eigenvalues. Denote $g \coloneqq Q\nabla f(x)$ and thus $g_i$ is the component of the gradient along the $i$-th eigenvector. 
Then we have the following theorem for $R_t$ with proof in Appendix~\ref{app:proof:Rt:N}.
\begin{theorem}[Sharpness under Gaussian Perturbation] \label{thm:Rt:N}
Under Gaussian perturbation, if we choose $\rho$ such that $1-t\rho^2 \lambda_i > 0$ holds for any $i$, then we have
    \begin{align*}
    R_t = \frac{1}{2t} \sum_{i=1}^d \left[\frac{(t\rho g_i)^2}{1-t\rho^2\lambda_i} - \log(1-t\rho^2\lambda_i)\right]. \numberthis \label{eq:Rt:N}
\end{align*}
\end{theorem}

We see that as $t\to 0$, we have $\lim_{t\to 0} R_t = \frac{\rho^2}{2} \sum_{i=1}^d \lambda_i = R_{\text{avg}}$, which is consistent with existing work \citep{wen2022does, universal}. As $t$ increases, the regularizer sensitivity $\phi_i(t)$ satisfies
\begin{align*}
    \phi_i(t) = \frac{\rho^2}{2(1-t\rho^2\lambda_i)}\left(\frac{t^2\rho^2g_i^2}{1-t\rho^2g_i}+1\right) > 0 \text{ for valid } \rho.
\end{align*}
It implies that the sensitivity of $R_t$ to $\lambda_i$ depends on $t$. When $t=0$, the sensitivity is the constant $\rho^2/2$, that is, each eigenvalue contributes the same to $R_t$. As $t$ increases, the sensitivity increases, meaning that $R_t$ will be more dominated by large eigenvalues.

\subsection{Ball Perturbation} \label{sec:Rt:B}

Apart from the Gaussian perturbation, we analyze the regularizer and explicit bias under the ball perturbation in this section. Since the Taylor expansion error of $f(x+\rho v)$ with $v\sim\mathcal{U}(\sqrt{d}\mathbb{B}^d)$ is $O(\rho^2d)$ (due to $\norm{v}^2 \leq d$), $F_t(x)$ can be decomposed into
\begin{align*}
     F_t(x) = f(x) + \underbrace{\frac{1}{t} \log \mathbb{E}_{\mathcal{U}(\sqrt{d}\mathbb{B}^d)}\left[\exp\left(t[\rho \nabla f(x)^{\top}v + \frac{\rho^2}{2} v^{\top} \nabla^2f(x)v]\right)\right]}_{\text{ Sharpness } R_t} + O(\rho^2 d)
\end{align*}
and we have the following theorem for $R_t$, whose complete derivation is in Appendix~\ref{app:proof:Rt:B}.
\begin{theorem}[Sharpness under Ball Perturbation] \label{thm:Rt:B}
Assume that $\norm{\nabla f(x)} < \infty$ and $\nabla^2 f(x)$ has bounded eigenvalues for any $x$ in our optimization trajectory. Under ball perturbation, $R_t$ is continuous and non-decreasing in $t$ for any $t<\infty$, and the regularizer sensitivity $\phi_i$ is continuous and non-decreasing in $\lambda_i$. Therefore, we analyze two extreme cases. When $t\to 0$, 
\begin{align}
      \lim_{t\to 0} R_t = \frac{\rho^2d}{2(d+2)} \sum_{i=1}^d \lambda_i, \label{eq:R0:B}
\end{align}
which recovers the sharpness of the vanilla zeroth-order methods $R_{\text{avg}}$, i.e., a simple average of eigenvalues. When $t\to \infty$, we have
\begin{align*}
   \lim_{t\to \infty} R_t := R_{\infty} = \max_{\norm{u} \leq \sqrt{d}} \underbrace{\rho g^{\top}u}_{\text{Slope penalty}} + \underbrace{\frac{\rho^2}{2}u^{\top}\Lambda u}_{\text{Curve penalty}}.\numberthis \label{eq:Rt:B:general} 
\end{align*}
\end{theorem}
Theorem~\ref{thm:Rt:B} indicates that as $t\to\infty$, $R_t$ \textit{pessimistically} regularizes the objective $f(x)$ so that we favor the \textit{flat} solution $\hat{x}$ where $f(\hat{x})$ has neither highly curved directions nor large slopes along the curved directions. We discuss the penalties Eq. (\ref{eq:Rt:B:general}) specifically in three regimes.

\textbf{Linear regime.} If $f(x)$ is piecewise-linear within the search space for the next iteration $x^{\prime}$, the curve penalty is zero and $R_{\infty}$ depends solely on $\max\{g^{\top}u:\norm{u}\leq\sqrt{d}\}= \sqrt{d}\norm{g} = \sqrt{d}\norm{\nabla f(x)}$ with $u^{\star}=\sqrt{d}g/\norm{g}$. Therefore, $F_t$ biases against the next iterations with steep slopes (gradients).

\textbf{Stationary regime.} If $f(x)$ has multiple local minima as candidates for the next iteration, the curve penalties for them are all zero and $R_{\infty}$ depends only on $\max\{u^{\top}\Lambda u: \norm{u}\leq\sqrt{d}\} = \sqrt{d}\max(\lambda_1, 0)$ with $u^{\star}=\sqrt{d}e_1$. Therefore, $F_t$ biases against next iterations with large curvature in \textit{any} direction.



\textbf{General case.} When both the curve and slope penalties are active, we use KKT conditions to solve Eq.~(\ref{eq:Rt:B:general}) in Appendix~\ref{app:proof:Rt:B:general}. We have that when $\nabla^2 f(x) \npreceq 0$, 
\begin{enumerate}
    \item Gradient–curvature co-alignment plays a critical role. Only eigen directions with nonzero gradient projection ($g_i \neq 0$) influence $R_t$, and the influence grows with both $|g_i|$ and $\lambda_i$.
    \item The largest positive eigenvalues dominate if the gradient points there, i.e., when $g$ has projections on the top-eigenvector(s), those eigenvalues have the largest impact on $R_t$.
\end{enumerate}
We make two comments based on the above results. \underline{First}, when $t\to 0$, our regularizer $R_t$ recovers $R_{\text{avg}}$ of the average-loss SAM objective under both Gaussian and uniform ball perturbations. As $t$ increases, regularizer sensitivity increases and thus the penalty from each eigenvalue changes from uniformity to dominance by $\lambda_{\text{max}}$. \underline{Second}, under ball perturbation (where the max-loss SAM objective is well-defined), as $t\to\infty$, our regularizer in the general case is consistent with the work of \citet{wen2022does} and consistent with \citet{universal} in the stationary regime. 

Furthermore, we discuss how hyperparameters such as $\rho$ and $d$ influence the effective choices of $t$ in Appendix~\ref{app:proof:choose_t}. In the following section, we present two low-dimensional examples that correspond to the linear and stationary regimes, respectively, to illustrate the effects of different biases introduced by $R_t$ in contrast to $R_{\text{avg}}$.

\subsection{Low-Dimensional Examples} \label{sec:toy}

We illustrate the benefit of ZEST and its sharpness notions through 2D examples for the linear and stationary regimes (details in Appendix~\ref{app:toy}). For the linear regime, we create a piecewise-linear loss function with one minimum (Figure~\ref{fig:piecewiselinear}). There are multiple routes to reach the minimum, some steep (with large slopes/gradient norms, lightly colored) and some flat (with small slopes, darkly colored). We observe that though both ZEST and the vanilla zeroth-order algorithm (MeZO by~\citet{mezo}) approach the minimum, ZEST identifies and chooses the flatter route (with darker planes) while MeZO chooses the steep trajectory. 

For the stationary regime, we present an $f$ with two local minima, $(\pm 1,0)$ such that $f(\pm 1,0)=0$ (Figure~\ref{fig:stationary}). Denoting the Hessian of $f$ at point $(x,y)$ as $H(x,y)$, we have the eigenvalues of $H(1,0)$ as $\{\frac{12}{5}, \frac{2}{5}\}$ and those of $H(-1,0)$ as $\{\frac{10}{5}, \frac{4}{5}\}$. Since the two minima have the same average of eigenvalues (trace), optimizers with sharpness defined as $R_{\text{avg}}$, such as MeZO, would treat these minima equally sharp. However, the fact that $\lambda_{\text{max}}[H(1,0)] > \lambda_{\text{max}}[H(-1,0)]$ indicates that there exist perturbation directions that substantially impair model utility if it reaches $(1,0)$, which should be avoided in critical applications. Noticeably, we observe that ZEST can avoid the pitfall of $(1,0)$ and arrive at $(-1,0)$ despite their identical loss value and Hessian trace. Being sensitive to $\lambda_{\text{max}}$, ZEST favors next iterates that are flat in \textit{any} direction, which is consistent with our analysis.

\begin{figure*}
    \centering
    \begin{subfigure}[b]{0.61\textwidth}
    \includegraphics[width=\textwidth]{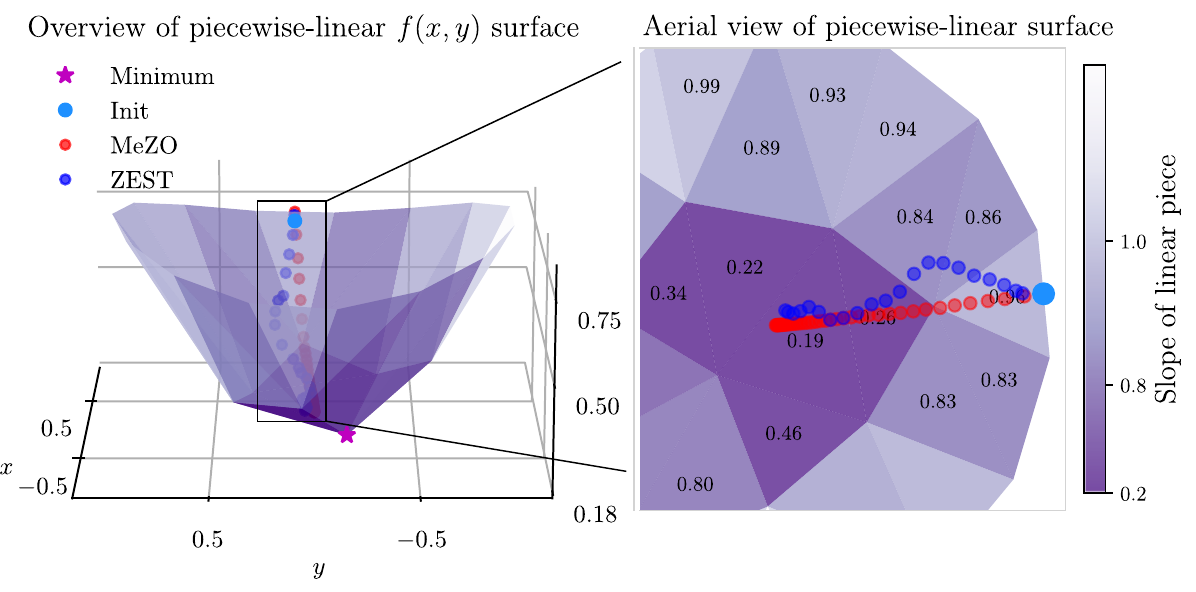}
    \caption{Convergence of different methods that prefer trajectories with different slopes. MeZO does not have a slope regularizer while ZEST identifies flat next iterations with smaller slopes (gradients). The color and the value of each triangle indicate its slope, with \textit{darker} indicating \textit{flatter}.}
    \label{fig:piecewiselinear}
    \end{subfigure}
    \hspace{0.3em}
    \begin{subfigure}[b]{0.35\textwidth} 
    \includegraphics[width=\textwidth]{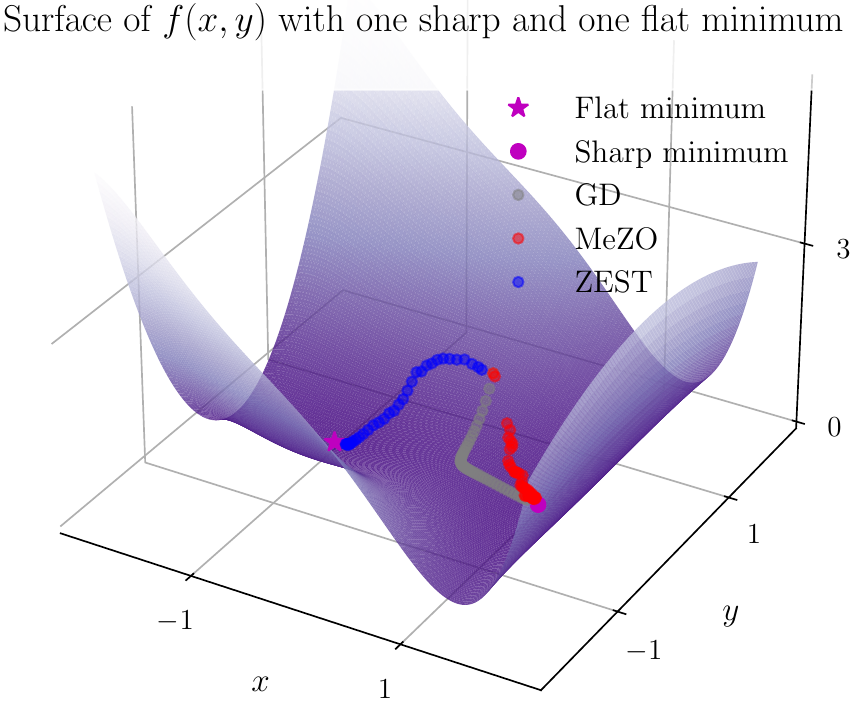}
    \caption{Convergence of different methods when two minima have the same loss value and average eigenvalues. GD and MeZO converge to the sharp minimum (with larger $\lambda_{\text{max}}$); ZEST converges to the flat one (smaller $\lambda_{\text{max}}$).}
    \label{fig:stationary}
    \end{subfigure}
    \caption{Convergence behaviors of different methods on examples for the (a) linear and (b) stationary regimes. It shows that (1) MeZO can make steep steps while ZEST identifies flat next iterations, and (2) MeZO can converge to minima with a large $\lambda_{\max}$ while ZEST explicitly biases against $\lambda_{\max}$.}
\end{figure*}

\section{Experiments} \label{sec:exp}

We conduct experiments on masked language models (LMs) (RoBERTa-base \citep{roberta}) on GLUE classification tasks in Section~\ref{sec:roberta} and autoregressive LMs (OPT-1.3B \citep{optpaper}) on multiple choice and generation tasks in Section~\ref{sec:opt}. We focus on many-shot fine-tuning with prompts, following prior zeroth-order literature \citep{mezo,dpzero,ZOO}. Across diverse tasks and model types, ZEST is a computationally and memory-efficient alternative to first-order approaches and outperforms the vanilla zeroth-order baseline MeZO. In Section~\ref{sec:flat}, we evaluate the flatness of ZEST solutions under multiple sharpness definitions. In Section~\ref{sec:sensitivity}, we discuss the effects of the tilting hyperparameter $t$ and provide practical guidance for choosing $t$.

\paragraph{Baselines.} For first-order baselines, we perform full-parameter fine-tuning to minimize ERM $f(x)$, the average-loss objective $\mathbb{E}[f(x+\epsilon)]$, the max-loss objective $\max_{\epsilon} f(x+\epsilon)$, and $t$-SAM $F_t(x)$. 
These objectives are solved via SGD, ESAM \citep{wen2022does}, SAM~\citep{sam_paper}, and TSAM \citep{tsam}, respectively.
For zeroth-order methods, we compare ZEST, which minimizes $F_t(x)$, with MeZO, which minimizes $\mathbb{E}[f(x+\epsilon)]$. The detailed implementations are in Appendix~\ref{app:exp:setup}. For TSAM and ZEST, we try $t\in\{1, 5, 20\}$ and selects the best value based on validation data. Note that TSAM with $t=0$ recovers ESAM (first-order), and ZEST with $t=0$ recovers MeZO (zeroth-order). We summarize the objectives, algorithms, and the memory complexities in Table~\ref{tab:methods}. 

We report the performance of both $\text{ZEST}_{\text{N}}$ (Option 1) and $\text{ZEST}_{\text{BC}}$ (Option 2), which use different update rules as in Algorithm~\ref{alg:zest} Line 9. Additionally, we highlight that ZEST has the same computational and memory complexity as vanilla zeroth-order methods since the cost of taking the exponential of a few losses is negligible. Therefore, the empirical memory efficiency and wallclock-time analysis in prior  works apply to ZEST (Appendix E.7, F.5, and F.6 of \citet{mezo}). 

\begin{table}[htbp]
\small\centering
\caption{Objective and memory cost of different methods. We follow the memory analysis in \citet{ZOO}.  $l$ is the layer index, $a_l$ denotes the stored activations for computing the backward gradients for layer $l$, and $|\cdot|$ denotes the dimension of the vector. We present the memory usage under ball perturbation when applicable since it is more costly than sampling from Gaussian.}
\vspace{-0.5em}
\begin{tabular}{llll}
\toprule
 Type & Objective & Method & Memory \\
\midrule
\multirow{4}{*}{\makecell{1st-\\order}} 
  & $f(x)$ & SGD & {$\sum_{l}\max(|a_l|,|x_l|)+|x|$} \\
  & $\mathbb{E}[f(x+\epsilon)]$ & ESAM \citep{wen2022does}  &  $\sum_{l}\max(|a_l|,|x_l|)+2|x|$  \\
  & $\max_{\epsilon} f(x+\epsilon)$ & SAM \citep{sam_paper} & {$\sum_{l}\max(|a_l|,|x_l|)+2|x|$} \\
  & $F_t(x)$ & TSAM \citep{tsam}& {$\sum_{l}\max(|a_l|,|x_l|)+(k+1)|x|$} \\
\midrule
\multirow{2}{*}{\makecell{0th-\\order}}
  & $\mathbb{E}[f(x+\epsilon)]$ & MeZO \citep{mezo} & \multirow{2}{*}{$2|x|$} \\
  & \multirow{1}{*}{$F_t(x)$ (\ref{eq:tsam})} & $\text{ZEST}$ (ours) & \\
\bottomrule
\end{tabular} \label{tab:methods}
\end{table}

\subsection{Masked Language Models} \label{sec:roberta}

We experiment on four types of classification tasks in the GLUE benchmark \citep{glue}, including sentiment classification, paraphrasing, topic classification, and natural language inference. Following prior work \citep{mezo,dpzero,ZOO}, we focus on the setting of many-shot fine-tuning with prompts where we sample 512 samples for each class. Since SAM is robust to label noise \citep{baek2024sam,tsam}, we additionally fine-tune on the noisy version of each dataset where the label noises are created by switching 30\% of the true labels uniformly at random to other labels (details in Appendix~\ref{app:exp}). 

\begin{table}[htbp]
\centering
\small\caption{Experiments on RoBERTa-Base (512 training examples per class). The objectives of each method are in Table~\ref{tab:methods}.}
\vspace{-0.5em}
\begin{tabular}{llcccccccc}
\toprule
 \multirow{2}{*}{Type} & \multirow{2}{*}{Method} & \textbf{SST-2} & \textbf{SST-5} & \textbf{QQP} & \textbf{MRPC} & \textbf{TREC} & \textbf{MNLI} & \textbf{SNLI} & \textbf{RTE} \\
 & & \multicolumn{2}{c}{sentiment cls.} & \multicolumn{2}{c}{paraphrase} & topic cls. & \multicolumn{3}{c}{natural language inference} \\
\midrule
\multirow{4}{*}{\makecell{1st-\\order}}
  &  SGD & 92.8 & 56.2 & 84.0 & 88.2 & 97.6 & 78.4 & 84.7 & 78.3 \\
  &  ESAM & 93.0 & 56.4& 84.3 & 88.5 & 97.8 & 78.4 & 85.3 & 79.4 \\
  & SAM & 93.2& 56.4& 84.8& \textbf{90.0} &97.8 & 79.3 & 85.4 & 80.1 \\
  & TSAM & \textbf{93.5} & \textbf{57.5}& \textbf{85.0}&89.2 & \textbf{98.0} & \textbf{79.5} & \textbf{85.8} & \textbf{80.5} \\
\midrule
\multirow{3}{*}{\makecell{0th-\\order}}
  & MeZO & 92.1 & 48.6 & 71.4 & 81.9 & 94.8 & 71.8& 78.2 & 72.9 \\
  & \cellcolor{pur!40}$\text{ZEST}_{\text{N}}$ & \cellcolor{pur!40}\textbf{92.2} & \cellcolor{pur!40}49.4 & \cellcolor{pur!40}71.6 & \cellcolor{pur!40}\textbf{83.6} & \cellcolor{pur!40}\textbf{95.6} & \cellcolor{pur!40}73.6 & \cellcolor{pur!40}\textbf{78.3} & \cellcolor{pur!40}\textbf{73.3} \\
  & \cellcolor{pur!40}$\text{ZEST}_{\text{BC}}$ & \cellcolor{pur!40}92.0 & \cellcolor{pur!40}\textbf{49.7} & \cellcolor{pur!40}\textbf{72.6} & \cellcolor{pur!40}81.6 & \cellcolor{pur!40}95.2 & \cellcolor{pur!40}\textbf{73.8} & \cellcolor{pur!40}78.2 & \cellcolor{pur!40}72.9 \\
\bottomrule
\end{tabular} \label{tab:roberta:clean}
\end{table}

\begin{table}[htbp]
\small\centering
\caption{Experiments on RoBERTa-Base (512 training examples per class with 30\% noisy labels). The objectives of each method are in Table~\ref{tab:methods}.}
\vspace{-0.5em}
\begin{tabular}{llcccccccc}
\toprule
 \multirow{2}{*}{Type} & \multirow{2}{*}{Method} & \textbf{SST-2} & \textbf{SST-5} & \textbf{QQP} & \textbf{MRPC} & \textbf{TREC} & \textbf{MNLI} & \textbf{SNLI} & \textbf{RTE} \\
 & & \multicolumn{2}{c}{sentiment cls.} & \multicolumn{2}{c}{paraphrase} & topic cls. & \multicolumn{3}{c}{natural language inference} \\
\midrule
\multirow{4}{*}{\makecell{1st-\\order}}
  &  SGD & 89.2 & 53.7& 73.8& 77.0& 96.2& 73.8& 78.1& 66.1\\
  &  ESAM & 89.9& 54.6& 79.5& 77.5& 96.2& 75.4& 79.2& 66.8\\
  & SAM & 91.1& \textbf{55.2}& 80.2& \textbf{78.9}& 96.2& \textbf{76.9}& 80.8& \textbf{68.6}\\
  & TSAM & \textbf{91.5} & \textbf{55.2}& \textbf{81.0} & 77.7 & \textbf{96.4} & 76.5& \textbf{81.4}& 67.5 \\
\midrule
\multirow{3}{*}{\makecell{0th-\\order}}
  & MeZO & 89.0& 44.7& 62.4& 67.2& 86.2& 60.3& 59.2& 59.9\\
  & \cellcolor{pur!40}$\text{ZEST}_{\text{N}}$ & \cellcolor{pur!40}\textbf{89.4} & \cellcolor{pur!40}\textbf{46.2}& \cellcolor{pur!40}\textbf{68.3}& \cellcolor{pur!40}68.6& \cellcolor{pur!40}\textbf{86.8}& \cellcolor{pur!40}\textbf{63.4}& \cellcolor{pur!40}\textbf{64.9}& \cellcolor{pur!40}61.4\\
  & \cellcolor{pur!40}$\text{ZEST}_{\text{BC}}$ & \cellcolor{pur!40}88.2& \cellcolor{pur!40}44.7& \cellcolor{pur!40}62.7& \cellcolor{pur!40}\textbf{68.9}& \cellcolor{pur!40}\textbf{86.8}& \cellcolor{pur!40}\textbf{63.4}&\cellcolor{pur!40}64.3& \cellcolor{pur!40}\textbf{61.7}\\
\bottomrule
\end{tabular} \label{tab:roberta:noisy}
\end{table}

On clean data, ZEST consistently outperforms MeZO by up to 1.7\% in accuracy (Table~\ref{tab:roberta:clean}), and on data with noisy labels, ZEST consistently outperforms MeZO by up to 5.9\% in accuracy (Table~\ref{tab:roberta:noisy}). On both clean and noisy data, SAM and TSAM consistently outperform ESAM, indicating the superiority of non-uniform regularizer sensitivity in $R_t$ as opposed to $R_{\text{avg}}$. We also observe that $\text{ZEST}_{\text{BC}}$ outperforms $\text{ZEST}_{\text{N}}$ on $3/8$ and $4/8$ tasks on clean and noisy data, respectively.

\subsection{Autoregressive Language Models} \label{sec:opt}

Apart from classification tasks, we experiment on multiple-choice and generation tasks with OPT-1.3B. For each dataset, we randomly sample 1000, 500, and 1000 examples for training, validation, and testing. From Table~\ref{tab:opt}, we observe that (1) TSAM and SAM consistently outperform ESAM, confirming the superiority of $R_t$ to $R_{\text{avg}}$; (2) ZEST consistently outperforms MeZO by up to 4.0\% in accuracy/F1 score and matches or outperforms first-order methods on multi-choice tasks.

\begin{table}[htbp]
\small\centering
\caption{Test accuracy/F1 of OPT-1.3B (1000 training samples). See Table~\ref{tab:methods} for method descriptions.}
\vspace{-0.5em}
\begin{tabular}{llcccc}
\toprule
  \multirow{2}{*}{Type} & \multirow{2}{*}{Method} & \textbf{COPA} & \textbf{ReCoRD} & \textbf{SQuAD} & \textbf{DROP} \\
 &  & \multicolumn{2}{c}{multiple choice} & \multicolumn{2}{c}{generation} \\
\midrule
\multirow{4}{*}{\makecell{1st-\\order}}
  &  SGD & 75.0 & 72.2 & 83.4 & 29.7 \\
  & ESAM & 76.0 & 72.5 & 83.7 & 31.2 \\
  & SAM & \textbf{77.0} & \textbf{72.7} & 84.3& \textbf{31.8} \\
  & TSAM & \textbf{77.0}& 72.1& \textbf{84.6}& 31.3\\
\midrule
\multirow{3}{*}{\makecell{0th-\\order}}
  &  MeZO & 74.0 & 72.4 & 78.8 & 25.2 \\
  & \cellcolor{pur!40}$\text{ZEST}_{\text{N}}$ & \cellcolor{pur!40}\textbf{78.0} & \cellcolor{pur!40}72.3 & \cellcolor{pur!40}\textbf{79.4} & \cellcolor{pur!40}25.5 \\
  & \cellcolor{pur!40}$\text{ZEST}_{\text{BC}}$ & \cellcolor{pur!40}77.0 & \cellcolor{pur!40}\textbf{72.5} & \cellcolor{pur!40}79.0 & \cellcolor{pur!40}\textbf{25.7} \\
\bottomrule
\end{tabular} \label{tab:opt}
\end{table}

We observe that $\text{ZEST}_{\text{BC}}$ and $\text{ZEST}_{\text{N}}$ perform on par in the above experiments. The potential reason is the use of small $k$, which makes the bias reduction from $O(1/k)$ to $O(1/k^2)$ not noticeable. We leave applying more advanced ratio estimates to ZEST to future work.

\subsection{Flatness of ZEST Solutions} \label{sec:flat}

\begin{wrapfigure}{r}{0.5\textwidth}
    \centering
    \vspace{-0.3in}
    \includegraphics[width=0.24\textwidth]{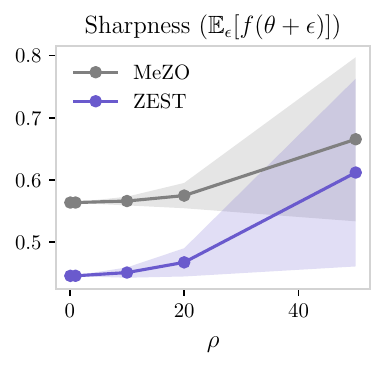}
    \includegraphics[width=0.24\textwidth]{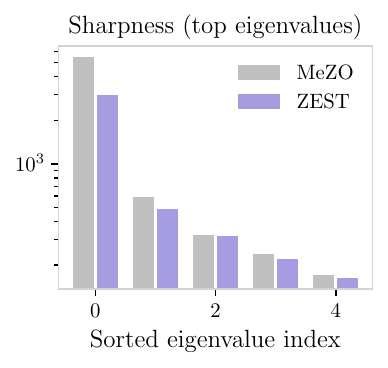}
    \caption{Sharpness of the solutions found by MeZO and ZEST on MRPC. Left: Scatters denote the average loss of the neighborhood among 500 perturbations, and the shade denotes the standard deviation. Right: Top-5 eigenvalues of Hessian.}
    \label{fig:flat_solution}
    \vspace{-0.3in}
\end{wrapfigure}

In this section, we evaluate the flatness of ZEST solutions in comparison to MeZO solutions. We compare their sharpness measurements under various definitions, including the average loss in the neighborhood of $x$ \citep{wen2023sharpness} and top-5 eigenvalues of the Hessian \citep{wen2022does}. In Figure~\ref{fig:flat_solution} (Left), we observe that under various neighborhood radii, the minimum found by ZEST has smaller average losses than that found by MeZO. In addition, the top-5 eigenvalues are all smaller than those of MeZO (Right). The same observation on more datasets is presented in Appendix~\ref{app:exp:result}.


\subsection{Sensitivity to $t$} \label{sec:sensitivity}

Though the generalization bounds for exponential tilting are presented in prior literature \citep{tsam,aminian2025generalization}, the optimal choice of $t$ is data-dependent. In practice, one needs to find the $t$ value that yields the best validation performance. In this section, we present the validation performance of RoBERTa-Base under $t=\{0, 1, 5, 20\}$ in Figure~\ref{fig:t_sensitivity}. The results show that multiple $t$ values yield superior performance to MeZO ($t=0$). Additionally, $t=1$ is a safe go-to choice for preliminary trials since it consistently yields superior or comparable performance to MeZO: $t=1$ matches or outperforms MeZO on $7/8$ settings; the only case when $t=1$ underperforms is by 0.1\%.

\begin{figure*}[h!]
    \centering
    \includegraphics[width=\textwidth]{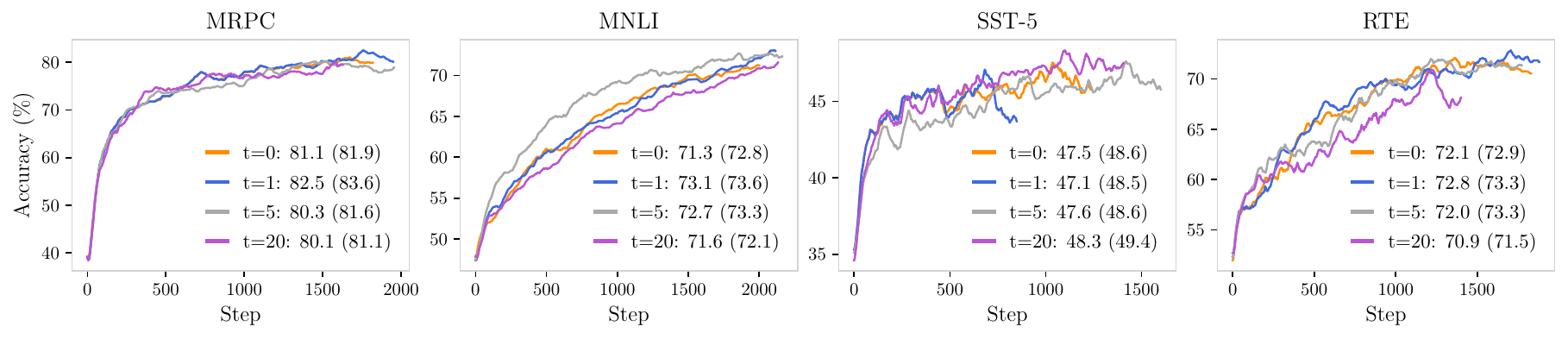}
    \includegraphics[width=\textwidth]{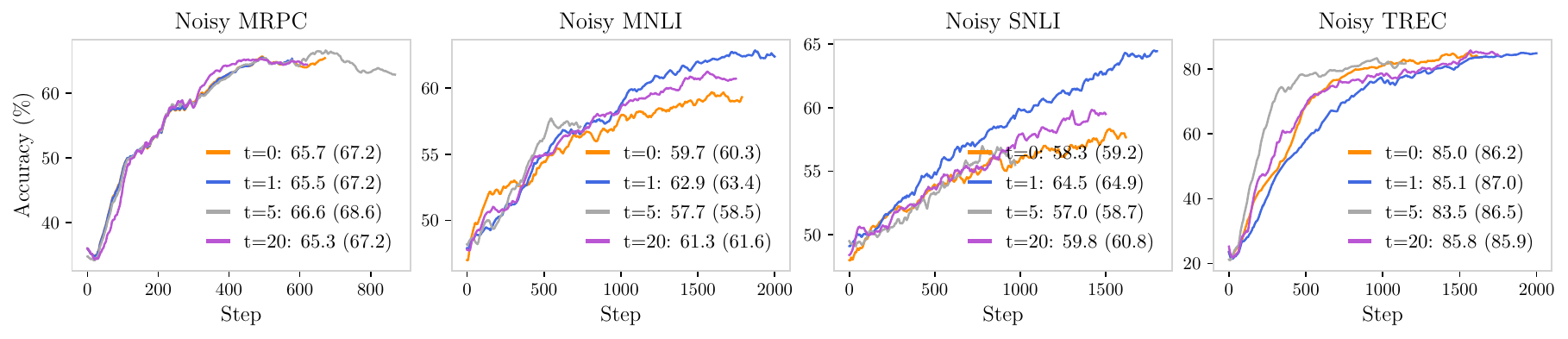}
    \caption{Validation accuracies of MeZO ($t=0$) and ZEST ($t=\{1,5,20\}$) on different datasets with clean labels (Upper) and 30\% noisy labels (Bottom). The x-axis denotes evaluation steps. On each dataset, we have $k=5$ sampled perturbations. The curves are smoothed for visualization, so we report the final smoothed accuracy and the final raw accuracy in the brackets. The results show that $t=1$ almost always outperforms MeZO ($t=0$): In the above plots, $t=1$ outperforms $t=0$ in raw accuracies by 1.7\%, 0.8\%, 0.4\%, 2.9\%, 5.7\%, 0.8\% and underperform by only 0.1\% on SST-5.}
    \label{fig:t_sensitivity}
\end{figure*}
\section{Conclusion}
We have introduced ZEST, a gradient-free optimization framework that unifies classic zeroth-order optimization with sharpness-aware minimization. By leveraging exponential tilting, ZEST optimizes for a continuous spectrum of objectives that smoothly interpolate between the standard average-loss zeroth-order objective and the worst-case min-max SAM formulation. Theoretically, we have characterized the sharpness bias induced by the tilted objective and demonstrate that ZEST can avoid minima of high curvatures that vanilla zeroth-order methods overlook. Empirically, ZEST preserves efficiency while consistently outperforming vanilla zeroth-order methods and, in many cases, first-order SAM variants on various downstream tasks. These demonstrate that ZEST provides a powerful bridge between zeroth-order optimization and sharpness-aware training, enabling gradient-free yet curvature-sensitive learning that generalizes better while remaining efficient.

\bibliographystyle{iclr2026}
\bibliography{cite}

\clearpage
\appendix
\section{Vanilla Zeroth-Order Gradient Estimate} \label{app:related}

In this section, we provide an additional introduction to zeroth-order optimization and the vanilla gradient estimate. 

In zeroth-order optimization, we estimate $\nabla f(x)$ using only function evaluations. A standard estimator is the two-point symmetric finite difference
\begin{align}
    G(x,\rho,u) := \frac{f(x + \rho u) - f(x - \rho u)}{2\rho} u,\label{eq:gZO}
\end{align}
where $u$ is a random direction sampled uniformly from the sphere $\sqrt{d}\mathbb{S}^{d-1}$ or Gaussian $\mathcal{N}(0,I_d)$, and $\rho > 0$ is a smoothing parameter. In the following, we abbreviate $\mathcal{B}\coloneqq\mathcal{U}(\sqrt{d}\mathbb{B}^{d})$, $\mathcal{S}\coloneqq\mathcal{U}(\sqrt{d}\mathbb{S}^{d-1})$, and $\mathcal{N}\coloneqq\mathcal{N}(0,I_d)$. We use $\mathbb{E}_{\mathcal{B}}$, $\mathbb{E}_{v\sim\mathcal{B}}$, and $\mathbb{E}_{v\sim\mathcal{U}(\sqrt{d}\mathbb{B}^d)}$ interchangeably when the meaning is clear from the context.

For sampling from the sphere, when $\rho \to 0$, the estimator is asymptotically unbiased since
$$\mathbb{E}_{u\sim\mathcal{S}} \left[ \frac{f(x + \rho u) - f(x - \rho u)}{2\rho}  u \right] \to \mathbb{E}_{u\sim\mathcal{S}} [ u u^\top ] \nabla f(x) = \nabla f(x).$$
When $\rho$ is general, the estimator corresponds to the gradient of a smoothed objective \citep{duchi2015optimal,dpzero}. Define
$$f_\rho(x) := \mathbb{E}_{v \sim \mathcal{B}}[f(x + \rho v)]$$
and by the divergence theorem in $\mathbb{R}^d$, 
$$\nabla_x f_\rho(x) = \mathbb{E}_{u \sim \mathcal{S}} \left[ G(x,\rho,u) \right].$$
Thus, the estimator in expectation is the gradient of a smoothed version of $f$ where the smoother is a uniform distribution on a ball. Similarly, for sampling from Gaussian, we have
\begin{align*}
    \nabla_x \mathbb{E}_{v\sim\mathcal{N}}[f(x+\rho v)] = \mathbb{E}_{v\sim\mathcal{N}}[G(x,\rho,v)].
\end{align*}
We can interpret the above results that updating using the vanilla zeroth-order gradient estimate optimizes for a smoothed objective of $f(x)$. By Taylor expansion, for $\pi\in\{\mathcal{S}, \mathcal{N}\}$, we have
\begin{align*}
    \mathbb{E}_{v\sim \pi}[f(x+\rho v)] & = f(x) + \mathbb{E}_{v\sim \pi}[\nabla f(x)^{\top} v] + \frac{\rho^2}{2}\mathbb{E}_{v\sim \pi}[v^{\top} \nabla^2 f(x) v] + \mathbb{E}_{v\sim \pi}[O(\rho^2\norm{v}^2)] \\
    & = f(x) + \frac{\rho^2}{2}\text{Tr}(\nabla^2 f(x)) + O(\rho^2d),
\end{align*}
which implies that the effective objective of vanilla zeroth-order optimization is the empirical loss $f(x)$ added by a regularizer $R_{\text{avg}} \propto \text{Tr}(\nabla^2 f(x))$.
\section{Proofs} \label{app:proof}

\subsection{Proof of Theorem~\ref{thm:ZEST} (Gaussian)}

\begin{proof}

By Stein's lemma \citep{stein}, for the $d$-dimensional random vector $v\sim \mathcal{N}(0, I_d)$ and a differentiable function $g$ for which $\mathbb{E}[g(v)v]$ and $\mathbb{E}[\nabla_v g(v)]$ both exist, we have 
\begin{align}
    \mathbb{E}_v[g(v)v] = \mathbb{E}_v[\nabla_v g(v)].\label{eq:Stein}
\end{align}
Therefore, we let $g(v) = e^{tf(x+\rho v)}$ and obtain
$$ \int_{v} \nabla_v (e^{tf(x+\rho v)}) p(v) dv = \int_{v} e^{tf(x+\rho v)} p(v) v dv $$
and thus
\begin{align}
    \int_{v} e^{tf(x+\rho v)-\norm{v}^2/2} \nabla_v f(x+\rho v) dv = \frac{1}{t} \int_{v}e^{tf(x+\rho v)-\norm{v}^2/2} v dv. \label{eq:Stein_app1}
\end{align}

Note that the gradient of $t$-SAM is 
\begin{align*}
    \nabla_x F_t(x) = \frac{\mathbb{E}_{v\sim\mathcal{N}}[e^{tf(x+\rho v)}\nabla f(x+\rho v)]}{\mathbb{E}_{v\sim\mathcal{N}}[e^{tf(x+\rho v)}]} = \frac{\int_{v} e^{tf(x+\rho v) - \norm{v}^2/2} \nabla f(x+\rho v) dv}{\int_{v} e^{tf(x+\rho v) - \norm{v}^2/2} dv}.
\end{align*}
Combining the above, we have
\begin{align*}
    \nabla_x F_t(x) & \overset{(a)}{=} \frac{\int_{v} e^{tf(x+\rho v) - \norm{v}^2/2} \nabla_v f(x+\rho v) dv}{\rho \int_{v} e^{tf(x+\rho v) - \norm{v}^2/2} dv} \\
    & \overset{(\ref{eq:Stein_app1})}{=} \frac{\int_{v}e^{tf(x+\rho v)-\norm{v}^2/2} v dv}{t\rho \int_{v} e^{tf(x+\rho v) - \norm{v}^2/2} dv} \\
    & = \frac{\mathbb{E}_{v\sim\mathcal{N}}[e^{tf(x+\rho v)} v]}{t\rho \mathbb{E}_{v\sim\mathcal{N}}[e^{tf(x+\rho v)}]} \numberthis\label{eq:T1} \\ 
    & = \frac{\frac{1}{2}(\mathbb{E}_{v\sim\mathcal{N}}[e^{tf(x+\rho v)} v] + \mathbb{E}_{v\sim\mathcal{N}}[e^{tf(x+\rho v)} v])}{t\rho \cdot \frac{1}{2}(\mathbb{E}_{v\sim\mathcal{N}}[e^{tf(x+\rho v)}] + \mathbb{E}_{v\sim\mathcal{N}}[e^{tf(x+\rho v)}])} \\
    & = \frac{1}{t\rho}\frac{\mathbb{E}_{v\sim\mathcal{N}}[e^{tf(x+\rho v)} v] + \mathbb{E}_{v\sim\mathcal{N}}[e^{tf(x+\rho (-v))} (-v)]}{\mathbb{E}_{v\sim\mathcal{N}(0, I_d)}[e^{tf(x+\rho v)}] + \mathbb{E}_{v\sim\mathcal{N}}[e^{tf(x+\rho (-v))}]} \\
    & = \frac{1}{t\rho}\frac{\mathbb{E}_{v\sim\mathcal{N}}[(e^{tf(x+\rho v)} - e^{tf(x-\rho v)}) v]}{\mathbb{E}_{v\sim\mathcal{N}}[e^{tf(x+\rho v)} + e^{tf(x-\rho v)}]} \numberthis
\end{align*}


where $(a)$ is due to $\nabla_x \phi(x + \rho v) = \nabla \phi(x + \rho v) = \frac{1}{\rho} \nabla_v \phi(x + \rho v)$ where $\nabla \phi(\cdot)$ denotes the gradient w.r.t. the input of function $\phi$. 

\end{proof}

\paragraph{Case of $\boldsymbol{t\to 0}$.} As $t\to 0$, we apply L'Hôpital's rule to obtain
\begin{align*}
    \lim_{t\rightarrow 0} \nabla_x F_t(x) & = \frac{\lim_{t\rightarrow 0} \mathbb{E}_{\mathcal{N}}[e^{tf(x+\rho v)}f(x+\rho v) v]}{\lim_{t\rightarrow 0} \rho \mathbb{E}_{\mathcal{N}}[e^{tf(x+\rho v)}] + t\rho \mathbb{E}_{\mathcal{N}}[e^{tf(x+\rho v)}f(x+\rho v)]} \\
    & = \mathbb{E}_{\mathcal{N}}\left[\frac{f(x+\rho v) v}{\rho}\right] \\
    & = \mathbb{E}_{\mathcal{N}}\left[\frac{f(x+\rho v) v}{2\rho}\right] + \mathbb{E}_{\mathcal{N}}\left[\frac{f(x+\rho (-v)) (-v)}{2\rho}\right] \\
    & = \mathbb{E}_{\mathcal{N}}\left[\frac{f(x+\rho v) - f(x-\rho v) }{2\rho} v \right],
\end{align*}
which is precisely the vanilla zeroth-order gradient estimator with Gaussian perturbation.

\subsection{Proof of Theorem~\ref{thm:ZEST} (Ball)}

\begin{proof}
Recall that under the uniform ball perturbation, the $t$-SAM gradient is
\begin{align}
    \nabla_x F_t(x) = \frac{\mathbb{E}_{v\sim\mathcal{U}(\sqrt{d}\mathbb{B}^d)}[e^{tf(x+\rho v)}\nabla f(x+\rho v)]}{\mathbb{E}_{v\sim\mathcal{U}(\sqrt{d}\mathbb{B}^d)}[e^{tf(x+\rho v)}]}.\label{eq:F_grad}
\end{align}
Denote $Z=\int_{\sqrt{d}\mathbb{B}^d}e^{tf(x+\rho v)}dv$. Then by definition, we have
    \begin{align}
        \mathbb{E}_{v\sim\mathcal{U}(\sqrt{d}\mathbb{B}^d)}[e^{tf(x+\rho v)}] = \frac{\int_{\sqrt{d}\mathbb{B}^d} e^{tf(x+\rho v)} dv}{\text{Vol}(\sqrt{d}\mathbb{B}^d)} = \frac{Z}{\text{Vol}(\sqrt{d}\mathbb{B}^d)} \label{eq:E[B]}
    \end{align}
    \begin{align}
        \mathbb{E}_{v\sim\mathcal{U}(\sqrt{d}\mathbb{B}^d)}[e^{tf(x+\rho v)}\nabla f(x+\rho v)] = \frac{\int_{\sqrt{d}\mathbb{B}^d}e^{tf(x+\rho v)}\nabla f(x+\rho v) dv}{\text{Vol}(\sqrt{d}\mathbb{B}^d)}, \label{eq:E[S]}
    \end{align}
and applying them to Eq.~(\ref{eq:F_grad}) gives us
$$ \nabla_x F_t(x) = \frac{\int_{\sqrt{d}\mathbb{B}^d}e^{tf(x+\rho v)}\nabla f(x+\rho v) dv}{Z}.$$
By change of variable, we have 
$$ \nabla_x \int_{\sqrt{d}\mathbb{B}^d} e^{tf(x+\rho v)} dv = \int_{\sqrt{d}\mathbb{B}^d} \nabla_x (e^{tf(x+\rho v)}) dv = \frac{1}{\rho} \int_{\sqrt{d}\mathbb{B}^d} \nabla_v (e^{tf(x+\rho v)}) dv. $$
According to the divergence theorem in higher dimensions, for a scalar field $\phi \in C^1: \mathbb{R}^d \to \mathbb{R}$ and a compact volume $\Omega \subset \mathbb{R}^d$ with piecewise smooth boundary $\partial\Omega$, we have
\begin{align}
   \int_{\Omega} \nabla \phi dV = \int_{\partial\Omega} \phi\textbf{n} dS \label{eq:divergence}
\end{align}
where $\textbf{n}$ is the outward unit normal to the point on $\partial\Omega$, given that both sides of the equation are integrable over their domains. Therefore, by letting $\phi(v) = e^{tf(x+\rho v)}$, $\Omega=\sqrt{d}\mathbb{B}^d$, and $\partial\Omega=\sqrt{d}\mathbb{S}^{d-1}$, we obtain
$$ \int_{\sqrt{d}\mathbb{B}^d} \nabla_v (e^{tf(x+\rho v)}) dv = \int_{\sqrt{d}\mathbb{S}^{d-1}} e^{tf(x+\rho u)} \frac{u}{\norm{u}} du = \frac{1}{\sqrt{d}} \int_{\sqrt{d}\mathbb{S}^{d-1}} e^{tf(x+\rho u)} u du. $$
Expanding the LHS gives us 
\begin{align*}
    \int_{\sqrt{d}\mathbb{B}^d} e^{tf(x+\rho v)} \nabla_v f(x+\rho v) dv = \frac{1}{t\sqrt{d}} \int_{\sqrt{d}\mathbb{S}^{d-1}} e^{tf(x+\rho u)} u du. \numberthis \label{eq:BS}
\end{align*}
Combining the above, we obtain
\begin{align*}
    \nabla_x F_t(x) & = \frac{1}{\rho Z} \int_{\sqrt{d}\mathbb{B}^d} e^{tf(x+\rho v)} \nabla_vf(x+\rho v) dv \\
    & \overset{(\ref{eq:BS})}{=} \frac{1}{t \rho \sqrt{d}Z}  \int_{\sqrt{d}\mathbb{S}^{d-1}} e^{tf(x+\rho u)} u du \\
    & \overset{(a)}{=} \frac{\text{Area}(\sqrt{d}\mathbb{S}^{d-1})}{t \rho\sqrt{d} Z} \mathbb{E}_{u\sim \mathcal{U}(\sqrt{d}\mathbb{S}^{d-1})ß}[e^{tf(x+\rho u)} u] \\
    & \overset{(b)}{=} \frac{\sqrt{d} \cdot \text{Vol} (\sqrt{d}\mathbb{B}^d)}{t \rho\sqrt{d} Z} \mathbb{E}_{u\sim \mathcal{U}(\sqrt{d}\mathbb{S}^{d-1})}[e^{tf(x+\rho u)} u] \\
    & \overset{(\ref{eq:E[B]})}{=} \frac{1}{t \rho} \frac{\mathbb{E}_{u\sim \mathcal{U}(\sqrt{d}\mathbb{S}^{d-1})}[e^{tf(x+\rho u)} u]}{\mathbb{E}_{v\sim\mathcal{U}(\sqrt{d}\mathbb{B}^d)}[e^{tf(x+\rho v)}]} \numberthis \label{eq:ZOg}
\end{align*}
where $(a)$ follows the definition of $\mathbb{E}_{u\sim \mathcal{U}(\sqrt{d}\mathbb{S}^{d-1})}[e^{tf(x+\rho u)} u]$ and $(b)$ is due to $\text{Area}(r \mathbb{S}^{d-1}) = \frac{d}{r} \cdot \text{Vol} (r \mathbb{B}^d)$, which gives us $\text{Area}(\sqrt{d}\mathbb{S}^{d-1}) = \sqrt{d} \cdot \text{Vol} (\sqrt{d}\mathbb{B}^d)$.

\paragraph{Case of $\boldsymbol{t\to 0}$.} As $t\to 0$, we apply L'Hôpital's rule to obtain
\begin{align*}
    \lim_{t\to 0} \nabla_x F_t(x) & = \frac{\lim_{t\to 0} \mathbb{E}_{\mathcal{S}}[e^{tf(x+\rho u)}f(x+\rho u) u]}{\lim_{t\to 0} \rho \mathbb{E}_{\mathcal{B}}[e^{tf(x+\rho v)}] + t\rho \mathbb{E}_{\mathcal{B}}[e^{tf(x+\rho v)}f(x+\rho v)]} \\
    & = \mathbb{E}_{\mathcal{S}}\left[\frac{f(x+\rho u) u}{\rho}\right] \\
    & = \mathbb{E}_{\mathcal{S}}\left[\frac{f(x+\rho u) u}{2\rho}\right] + \mathbb{E}_{\mathcal{S}}\left[\frac{f(x+\rho (-u)) (-u)}{2\rho}\right] \\
    & = \mathbb{E}_{\mathcal{S}}\left[\frac{f(x+\rho u) - f(x-\rho u) }{2\rho} u \right],
\end{align*}
which recovers the vanilla zeroth-order gradient estimator in Eq.~(\ref{eq:gZO}).

\end{proof}

\subsection{Reusing Sphere Perturbations}

In Theorem~\ref{thm:ZEST}, we use Eq.~(\ref{eq:ss}) to approximate Eq.~(\ref{eq:sb}). The rationale of this choice is the fact that most of the volume of a high-dimensional ball is concentrated near its boundary. As specified in Lemma~\ref{lemma:concentration}, $\mathbb{E}[\norm{v}] \approx \sqrt{d}$ and $\text{Var}(\norm{v}) \approx \frac{1}{3d}$ for $d\gg 1$, which agrees with what we encounter in practice. This motivates us to use the same sampled perturbations and the computed losses to compute both the numerator and the denominator, which thus gives ZEST the same computational workload as the vanilla zeroth-order optimization method. 


\begin{lemma}[Measure of Concentration] \label{lemma:concentration}
    For a random point uniformly sampled from a ball with radius $\sqrt{d}$, its norm $\norm{v}$ satisfies
    \begin{align}
        \mathbb{E}[\norm{v}] = \sqrt{d}\left(1-\frac{1}{d+1}\right)
    \end{align}
    \begin{align}
        \text{Var}(\norm{v}) = \frac{d^2}{d+2} - \frac{d^3}{(d+1)^2} \overset{d\gg 1}{\approx} \frac{1}{3d}
    \end{align}
\end{lemma}

\begin{proof}
    Denote $q_{\norm{v}}(r)$ as the probability density of the event $\norm{v}=r$, which is proportional to the surface area of the sphere $r\mathbb{S}^{d-1}$:
    $$q_{\norm{v}}(r) = \frac{dr^{d-1}}{d^{d/2}}, \; 0\leq r\leq \sqrt{d}.$$
    Then the first and second moment of $\norm{v}$ are 
    $$ \mathbb{E}[{\norm{v}}] = \int_{0}^{\sqrt{d}} r \cdot q_{\norm{v}}(r) dr = \frac{d}{d^{d/2}} \int_{0}^{\sqrt{d}} r^d dr = \sqrt{d} \frac{d}{d+1}$$
    $$ \mathbb{E}[{\norm{v}}^2] = \int_{0}^{\sqrt{d}} r^2 \cdot q_{\norm{v}}(r) dr = \frac{d}{d^{d/2}} \int_{0}^{\sqrt{d}} r^{d+1} dr = \frac{d^2}{d+2} $$
    and thus
    $$ \text{Var}(\norm{v}) = \frac{d^2}{d+2} - \frac{d^3}{(d+1)^2} = \frac{d^2}{(d+2)(d+1)^2} \overset{t\to\infty}{\longrightarrow} \lim_{d\to\infty} \frac{1}{3d+4}.$$
\end{proof}

\subsection{Bias-Corrected Ratio Estimate} \label{app:proof:BC}

In this section, we derive the bias-corrected estimate with bias $O(1/k^2)$. Recall that in each iteration, we sample $k$ perturbations and compute $a_i^+=e^{tf(x+\rho v_i)}$, $a_i^-=e^{tf(x-\rho v_i)}$, and $Z=\sum_{i=1}^k a_i^+ + a_i^-$. We aim to approximate
$$\frac{\mathbb{E}[A]}{\mathbb{E}[B]}, \; \text{ with samples } \; A_i = (a_i^+ - a_i^-) v_i \;\; \text{ and } \; B_i = a_i^+ + a_i^-, \; i\in[k].$$
In the following, we show that making up the bias in the naive plug-in leads to the following estimate:
\begin{align*}
    t\rho  G_{\text{BC}}^k = \sum_{i=1}^k \left\{1+\frac{k}{k-1}[\bar{a}_i^+ + \bar{a}_i^- - \sum_{i=1}^k (\bar{a}_i^+ + \bar{a}_i^-)^2] \right\} (\bar{a}_i^+ - \bar{a}_i^-) v_i.
\end{align*}
\begin{proof}
    Define the function $g(x,y)=\frac{x}{y}$ with $x\in\mathbb{R}^d$ and $y\in\mathbb{R}$. Let $\bar{A}=\frac{1}{k}\sum_{i=1}^k A_i$, $\bar{B}=\frac{1}{k}\sum_{i=1}^k B_i$, $\mu_A = \mathbb{E}[A]$, and $\mu_B = \mathbb{E}[B]$. We expand $g(\bar{A},\bar{B})$ around the point 
$(\mu_A, \mu_B)$ and have 
\begin{align*}
    g(\bar{A}, \bar{B}) & \approx g(\mu_A, \mu_B) + g_{\bar{A}}^{\top} (\bar{A} - \mu_A) + g_{\bar{B}} (\bar{B} - \mu_B) \\
& \quad + \frac{1}{2}[(\bar{A} -\mu_A)^{\top} g_{\mu_{A}\mu_{A}}(\bar{A} -\mu_A) + 2 (\bar{A} -\mu_A)^{\top} g_{\mu_{A}\mu_{B}} (\bar{B} - \mu_B) + g_{\mu_{B}\mu_{B}}(\bar{B}-\mu_B)^2]
\end{align*}
where $g_{\bar{A}} = \frac{\partial g(\bar{A},\bar{B})}{\partial \bar{A}}$, $g_{\bar{B}} = \frac{\partial g(\bar{A},\bar{B})}{\partial \bar{B}}$, $g_{\mu_{A}\mu_{A}} = \frac{\partial^2 g(\mu_A,\mu_B)}{\partial \mu_{A}^2}=0$, $g_{\mu_{B}\mu_{B}} = \frac{\partial^2 g(\mu_A,\mu_B)}{\partial \mu_{B}^2} = \frac{2\mu_{A}}{\mu_{B}^3}$, and $g_{\mu_{A}\mu_{B}} = \frac{\partial^2 g(\mu_A,\mu_B)}{\partial \mu_{A} \partial\mu_{B}} = -\frac{1}{\mu_{B}^2}$. Applying them to the approximate equality and taking the expectation on both sides, we have
\begin{align*}
    \mathbb{E}\left[\frac{\bar{A}}{\bar{B}}\right] & \approx \frac{\mu_A}{\mu_B} - \frac{1}{\mu_{B}^2} \mathbb{E}[(\bar{A} -\mu_A)^{\top} (\bar{B} - \mu_B)] + \frac{\mu_{A}}{\mu_{B}^3}\mathbb{E}[(\bar{B}-\mu_B)^2] \\
    & = \frac{\mu_A}{\mu_B} - \frac{1}{\mu_{B}^2} \text{Cov}(\bar{A},\bar{B}) + \frac{\mu_{A}}{\mu_{B}^3}\text{Var}(\bar{B}) \\
    & = \frac{\mathbb{E}[A]}{\mathbb{E}[B]} - \frac{1}{\mu_{B}^2k} \text{Cov}(A,B) + \frac{\mu_{A}}{\mu_{B}^3k}\text{Var}(B).
\end{align*}
Therefore, we use
$$ t\rho  G_{\text{BC}}^k = \frac{\bar{A}}{\bar{B}} + \frac{1}{k(k-1)} \left[\frac{\sum_{i=1}^k(A_i-\bar{A})(B_i-\bar{B})}{\bar{B}^2} - \frac{\bar{A}\sum_{i=1}^k(B_i - \bar{B})^2}{\bar{B}^3}\right].$$
Denote $a_i^+=e^{tf(x+\rho v_i)}$, $a_i^-=e^{tf(x-\rho v_i)}$ and thus $A_i=(a_i^+-a_i^-)v_i$ and $B_i=a_i^++a_i^-$. In practice, we record $Z=\sum_{i=1}^k a_i^+ + a_i^- = k\bar{B}$ and work with the normalized values $\bar{a}_i^+ \coloneqq a_i^+/Z$ and $\bar{a}_i^- \coloneqq a_i^-/Z$ for numerical stability. So we re-express $t\rho  G_{\text{BC}}^k$ with $A_i'\coloneqq (\bar{a}_i^+ - \bar{a}_i^-) v_i$, $B_i'\coloneqq \bar{a}_i^+ + \bar{a}_i^-$, and thus $\bar{A}' \coloneqq \frac{1}{k} \sum_{i=1}^k A_i'$ as
\begin{align*}
    t\rho  G_{\text{BC}}^k & = \sum_{i=1}^k A_i' + \sum_{i=1}^k\frac{(kB_i'-1)}{k-1}(A_i'-\bar{A}') - \bar{A}' \sum_{i=1}^k \frac{(kB_i' - 1)^2}{k-1} \\
    & = \sum_{i=1}^k \left\{1+\frac{k}{k-1}[B_i'- \sum_{i=1}^k (B_i')^2] \right\}A_i'  \\
\end{align*}
\end{proof}



\subsection{Proof of Theorem~\ref{thm:Rt:N}}  \label{app:proof:Rt:N}
\begin{proof}
Recall that the Hessian $\nabla^2 f(x)$ is written as $\nabla^2 f(x) = Q^{\top}\Lambda Q$, where the orthogonal $Q$ has columns $\{e_1, \ldots, e_d\}$ that are ordered eigenvectors of $\nabla^2 f(x)$, and $\Lambda = \text{diag}(\lambda_1, \lambda_2, \ldots, \lambda_d)$ where $\lambda_1\geq \ldots \geq \lambda_d$ are the order eigenvalues of $\nabla^2 f(x)$. Observe that $u \coloneqq Qv$ has the same distribution as $v$ since Gaussian is rotation-invariant. Denote $g \coloneqq Q\nabla f(x)$ where $g_i$ is the component of the gradient along the $i$-th eigenvector. Then we have

\begin{align*}
   & \quad \mathbb{E}_{v\sim\mathcal{N}}\left[\exp\left(t(\rho \nabla f(x)^{\top}v + \frac{\rho^2}{2} v^{\top} \nabla^2f(x)v)\right)\right] \\
    & = \mathbb{E}_{u\sim\mathcal{N}}\left[\exp\left(t(\rho g^{\top}u + \frac{\rho^2}{2} u^{\top} \Lambda u)\right)\right] \\
    & = (2\pi)^{-d/2} \int \exp\left(t(\rho g^{\top}u + \frac{\rho^2}{2} u^{\top} \Lambda u) - \frac{1}{2}\sum_{i=1}^d u_i^2\right) du \\
    & = (2\pi)^{-d/2} \int \exp\left(t\rho \sum_{i=1}^d g_i u_i +  \sum_{i=1}^d \frac{t\rho^2 \lambda_i - 1}{2} u_i^2\right) du \\
    & = \prod_{i=1}^d \int \frac{1}{\sqrt{2\pi}} \exp\left(-\frac{1-t\rho^2 \lambda_i}{2} u_i^2 + t\rho g_i u_i\right) du_i \\
    \displaybreak\\
    & = \prod_{i=1}^d \int \frac{1}{\sqrt{2\pi}} \exp\left(\left[-\frac{1-t\rho^2 \lambda_i}{2} \left(u_i - \frac{t\rho g_i}{1-t\rho^2\lambda_i}\right)^2 + \frac{(t\rho g_i)^2}{2(1-t\rho^2\lambda_i)}\right]\right) du_i \\
    & = \prod_{i=1}^d  \exp\left(\frac{(t\rho g_i)^2}{2(1-t\rho^2\lambda_i)}\right) \int \frac{1}{\sqrt{2\pi}} \exp\left(\left[-\frac{1-t\rho^2 \lambda_i}{2} \left(u_i - \frac{t\rho g_i}{1-t\rho^2\lambda_i}\right)^2 \right]\right) du_i \\
    & = \prod_{i=1}^d  \frac{\exp\left(\frac{(t\rho g_i)^2}{2(1-t\rho^2\lambda_i)}\right)}{\sqrt{1-t\rho^2\lambda_i}} \underbrace{\int \frac{\sqrt{1-t\rho^2\lambda_i}}{\sqrt{2\pi}} \exp\left(\left[-\frac{1-t\rho^2 \lambda_i}{2} \left(u_i - \frac{t\rho g_i}{1-t\rho^2\lambda_i}\right)^2 \right]\right) du_i}_{=1} \\
    & = \frac{\exp\left(\sum_{i=1}^d \frac{(t\rho g_i)^2}{2(1-t\rho^2\lambda_i)}\right)}{\prod_{i=1}^d \sqrt{1-t\rho^2\lambda_i}}
\end{align*}

where $du$ denotes the Lebesgue measure on $\mathbb{R}^d$. Note that it is required for any $i$, $1-t\rho^2 \lambda_i > 0$, i.e., choose $t\rho^2 < \frac{1}{\lambda_{\text{max}}}$ if $\lambda_{\text{max}} > 0$. The regularizer is thus
\begin{align*}
    R_t & = \frac{1}{t} \sum_{i=1}^d \frac{(t\rho g_i)^2}{2(1-t\rho^2\lambda_i)} - \frac{1}{2t} \sum_{i=1}^d  \log(1-t\rho^2\lambda_i) \\
    & = \frac{1}{2t} \sum_{i=1}^d \left[\frac{(t\rho g_i)^2}{1-t\rho^2\lambda_i} - \log(1-t\rho^2\lambda_i)\right].
\end{align*}

When $t\to 0$, we apply L’Hôpital’s rule to obtain 
\begin{align*}
    \lim_{t\to 0} R_t = - \sum_{i=1}^d \lim_{t\to 0} \frac{d(\log(1-t\rho^2\lambda_i))/dt}{d(2t)/dt} = - \sum_{i=1}^d \lim_{t\to 0} \frac{-\rho^2\lambda_i}{2(1-t\rho^2\lambda_i)}  = \frac{\rho^2}{2} \sum_{i=1}^d \lambda_i.
\end{align*}
\end{proof}

\subsection{Proof of Theorem~\ref{thm:Rt:B}}  \label{app:proof:Rt:B}

We first present the proof of a useful lemma below.
\begin{lemma}[Laplace's Principle \citep{laplace}] \label{lemma:laplace}
    Let $\mathcal{M}$ be a Lebesgue-measurable subset of $d$-dimensional Euclidean space $\mathbb{R}^d$ and let $\varphi: \mathbb{R}^d \to \mathbb{R}$ be a measurable function with $\int_{\mathcal{M}} e^{-\varphi(x)}\,dx<\infty$. Then
$$ \lim_{t \to \infty }{\frac{1}{t}}\log \int_{\mathcal{M}}e^{-t \varphi (x)}\,dx=-\mathop{\mathrm {ess\,\; inf}}_{x\in \mathcal{M}}\varphi (x).$$
where $\mathrm{ess\,\;inf}$ denotes essential infimum.
\end{lemma}
\begin{proof}
    Denote $m\coloneqq \mathrm{ess\,\;inf}_{x\in\mathcal{M}} \varphi(x)$, fix $\varepsilon>0$ , and set $E_{\varepsilon} \coloneqq \{x\in\mathcal{M}: \varphi(x)< m + \varepsilon\}$. By definition, $E_{\varepsilon}$ has positive measure. Therefore,
    $$ \int_{\mathcal{M}}e^{-t\varphi} \,dx \geq \int_{E_{\varepsilon}} e^{-t\varphi}\,dx \geq |E_{\varepsilon}|e^{-t(m+\varepsilon)}$$
    and hence
    $$\underset{t\to\infty}{\lim\,\inf} \frac{1}{t} \log\int_{\mathcal{M}}e^{-t\varphi} \, dx = \underset{t\to\infty}{\lim\,\inf} \frac{\log |E_{\varepsilon}|}{t} -(m+\varepsilon) \geq -(m+\varepsilon).$$
    Let $\varepsilon\to 0$ and we have LHS equal to $-m$. On the other hand, we split $\mathcal{M}=E_{\varepsilon} \cup E_{\varepsilon}^c$. Since $\varphi\geq m$ on $E_{\varepsilon}$, we have
    $$\int_{E_{\varepsilon}}e^{-t\varphi}\,dx\leq |E_{\varepsilon}|e^{-tm}.$$
    We have $\varphi\geq m+\varepsilon$ on $E_{\varepsilon}^c$, so for $t\geq 1$, $-t\varphi \leq -(t-1)(m+\varepsilon) - \varphi$ and
    $$\int_{E_{\varepsilon}^c}e^{-t\varphi}dx\leq e^{-(t-1)(m+\varepsilon)}\int_{\mathcal{M}}e^{-\varphi} \,dx = Ce^{-(t-1)(m+\varepsilon)}$$
    for some $C < \infty$. Therefore, 
    $$ \underset{t\to\infty}{\lim\,\sup} \frac{1}{t} \log\int_{\mathcal{M}}e^{-t\varphi} \,dx \leq \underset{t\to\infty}{\lim\,\sup} \left\{-m + \frac{\log(|E_{\varepsilon}| + Ce^{-t\varepsilon+m+\varepsilon})}{t} \right\} = -m - \varepsilon. $$
    Let $\varepsilon\to 0$ and we have LHS equal to $-m$. We combine it with the lower bound to conclude that the limit is equal to $-m$.
\end{proof}

In the following, we prove the statements in Theorem~\ref{thm:Rt:B}. 
\begin{proof}
Recall that we denote $g=Q\nabla^2 f(x)$ and $\Lambda = \text{diag}(\lambda_1, \lambda_2, \ldots, \lambda_d)$ defined as before. Assume that $\norm{\nabla f(x)} < \infty$ and $\nabla^2 f(x)$ has bounded eigenvalues for any $x$ in our optimization trajectory. 

We denote $X=\rho g^{\top}u + \frac{t\rho^2}{2}u^{\top} \Lambda u$ with $X < \infty$, and we thus have $\mathbb{E}[\exp(tX)]<\infty$ for any $t < \infty$. Since $h(X) = \frac{1}{t} \log(\mathbb{E}[\exp(tX)])$ is continuous for $t\in\{t:\mathbb{E}[\exp(tX)]<\infty\}$, $h(X)$ is continuous and non-decreasing in $t$ for any $t<\infty$. Furthermore, the regularizer sensitivity is
\begin{align*}
    \phi_i = \frac{1}{t} \cdot \frac{1}{\mathbb{E}[\exp(tX)]} \cdot \frac{\partial \mathbb{E}[\exp(tX)]}{\partial \lambda_i} = \frac{\rho^2}{2} \frac{\mathbb{E}[\exp(t[\rho g^{\top}u + \frac{t\rho^2}{2}u^{\top} \Lambda u])u_i^2]}{\mathbb{E}[\exp(t[\rho g^{\top}u + \frac{t\rho^2}{2}u^{\top} \Lambda u])]}.
\end{align*}
It is continuous and non-decreasing in $\lambda_i$ due to
\begin{align*}
    \frac{\partial \phi_i}{\partial \lambda_i} = \frac{t\rho^4}{4} \cdot \frac{\mathbb{E}[u_i^4 e^{tX}]\mathbb{E}[e^{tX}] - (\mathbb{E}[u_i^2e^{tX}])^2}{(\mathbb{E}[e^{tX}])^2} \geq 0
\end{align*}
where the inequality follows the Cauchy-Schwarz inequality $(\mathbb{E}[AB])^2 \leq \mathbb{E}[A^2] \mathbb{E}[B^2]$. Therefore, it suffices to analyze $R_t$ and $\phi_i(t)$ under the two extreme cases that $t\to 0$ and $\infty$. Recall that we have
\begin{align*}
    R_t & = \frac{1}{t} \log \left(\int \exp \left(t\rho g^{\top}u + \frac{t\rho^2}{2}u^{\top} \Lambda u\right)\, d\mu(u)\right) \quad \text{ with } u=Qv \sim \mathcal{U}(\sqrt{d}\mathbb{B}^d) \\
     & = \frac{1}{t} \log \left( \frac{1}{\text{Vol}(\sqrt{d}\mathbb{B}^d)} \int_{\norm{u}\leq \sqrt{d}} \exp \left(t\rho g^{\top} u + \frac{t\rho^2}{2}u^{\top} \Lambda u\right)\, du\right)\numberthis \label{eq:Rt:B}
\end{align*}
\paragraph{Case of $\boldsymbol{t\to 0}$.} When $t\to 0$, we apply L’Hôpital’s rule to Eq.~(\ref{eq:Rt:B}) and obtain
\begin{align*}
      \lim_{t\to 0} R_t & =  \lim_{t\to 0} \frac{\int_{\norm{u}\leq \sqrt{d}} \nabla_t \left[\exp \left(t\rho g^{\top}u + \frac{t\rho^2}{2}u^{\top} \Lambda u\right)\right] \,du}{\int_{\norm{u}\leq \sqrt{d}} \exp \left(t\rho g^{\top} u +\frac{t\rho^2}{2}u^{\top} \Lambda u\right)\, du} \\
      & =  \lim_{t\to 0} \frac{\int_{\norm{u}\leq \sqrt{d}} \exp \left(t\rho g^{\top}u + \frac{t\rho^2}{2}u^{\top} \Lambda u\right) \left(\rho g^{\top}u + \frac{\rho^2}{2}u^{\top} \Lambda u\right) \,du}{\int_{\norm{u}\leq \sqrt{d}} \exp \left(t\rho g^{\top} u +\frac{t\rho^2}{2}u^{\top} \Lambda u\right)\, du} \\
      & =  \lim_{t\to 0} \frac{\int_{\norm{u}\leq \sqrt{d}} \rho g^{\top}u + \frac{\rho^2}{2} u^{\top} \Lambda u \,du}{\text{Vol}(\sqrt{d}\mathbb{B}^d)} \\
      & =  \frac{\rho^2}{2\text{Vol}(\sqrt{d}\mathbb{B}^d)} \lim_{t\to 0} \int_{\norm{u}\leq \sqrt{d}} u^{\top} \Lambda u \,du \\
      & = \frac{\rho^2d}{2(d+2)} \sum_{i=1}^d \lambda_i
\end{align*}

where the last step is due to 
$$ \int_{\norm{u}\leq \sqrt{d}} u^{\top}\Lambda u\, du = \int_{\norm{u}\leq \sqrt{d}} \left(\sum_{i=1}^d \lambda_i u_i^2\right) \,du 
     = \sum_{i=1}^d \lambda_i \int_{\norm{u}\leq \sqrt{d}} u_i^2 \,du $$
and
$$ \int_{\norm{u}\leq \sqrt{d}} u_i^2 \,du = \frac{1}{d} \int_{\norm{u}\leq \sqrt{d}} \norm{u}^2 \, du = \frac{d}{2(d+2)} \text{Vol}(\sqrt{d}\mathbb{B}^d).$$

\paragraph{Case of $\boldsymbol{t\to \infty}$.} When $t\to \infty$, we apply Laplace's principle \citep{laplace} that for a Lebesgue-measurable set $\mathcal{M}\in\mathbb{R}^d$ and a measurable function $\varphi: \mathbb{R}^d \to \mathbb{R}$ that satisfy $ \int_{\mathcal{M}}e^{\varphi(x)}\,dx<\infty$, we have
$$\lim_{t \to\infty}{\frac{1}{t}}\log \int_{\mathcal{M}} e^{t\varphi(x)}\,dx = \max_{x\in \mathcal{M}} \varphi(x).$$
Let $\varphi(u)=\rho g^{\top}u +\frac{\rho^2}{2}u^{\top}\Lambda u $ and $\mathcal{M}=\sqrt{d}\mathbb{B}^d$. Since $\mathcal{M}$ is measurable and $\varphi(x)\leq \rho\sqrt{d}\norm{a} +\frac{\rho^2d}{2}\max(\lambda_{\text{max}}, 0)$, the integrability condition satisfies, and we have 
\begin{align*}
   \lim_{t\to \infty} R_t & = \lim_{t\to \infty} \frac{1}{t} \log\left(\frac{1}{\text{Vol}(\sqrt{d}\mathbb{B}^d)}\right) + \lim_{t\to \infty} \frac{1}{t} \log \left( \int_{\mathcal{M}} e^{t\varphi(u)} du\right) \\
   & = \max_{\norm{u} \leq \sqrt{d}} \varphi(u).
\end{align*}
\end{proof}

\subsection{General Regime ($t\to\infty$)} \label{app:proof:Rt:B:general}

Recall that we work in the Hessian eigenbasis with $\Lambda=\mathrm{diag}(\lambda_1,\ldots,\lambda_d)$ and $g=Q\nabla f(x)$. In the general regime where both the slope and curve penalties are active, we use KKT conditions to solve the maximization problem with an inequality constraint 
\begin{align*}
\max_{u:\,\|u\|\leq\sqrt{d}}\ \varphi(u)=\rho\,g^\top u+\frac{\rho^2}{2}\,u^\top \Lambda u.
\end{align*}
From the Lagrangian
$$ \mathcal{L}(u, \omega) = \rho a^{\top}u + \frac{\rho^2}{2} u^{\top}\Lambda u - \omega(u^{\top}u - d), \quad \omega\geq 0,$$
we have
\begin{align*}
    \rho g + \rho^2\Lambda u - 2\omega u = 0 &\Longleftrightarrow (\rho^2\Lambda - 2\omega I)u = -\rho g  \\
    \omega(\|u\| - \sqrt{d}) &= 0  \\
    \|u\| &\leq \sqrt{d}   \\
    2\omega I - \rho^2 \Lambda &\succeq 0 
\end{align*}
by stationarity, complementary slackness, primal feasibility, and dual feasibility, respectively.

\paragraph{Interior case.} When $\nabla^2 f(x)\preceq 0$ and the unconstrained maximizer is feasible, we take $\omega=0$ and thus $u^\star=-(1/\rho)\Lambda^{-1}g$ and
\begin{align*}
R_{\infty} = \varphi(u^\star)=-\frac{1}{2}\, g^\top\Lambda^{-1}g
= -\frac{1}{2}\sum_{i=1}^d\frac{g_i^2}{\lambda_i}\quad(\lambda_i\le 0),
\end{align*}
which indicates that making $\lambda_i$ more negative reduces the penalty. We also have the regularizer sensitivity that increases as $\lambda_i$ increases:
\begin{align*}
\phi_i = \frac{\partial R_{\infty}}{\partial \lambda_i}
= -\frac{1}{2}\,g_i^2\,\frac{\partial}{\partial \lambda_i}\left(\frac{1}{\lambda_i}\right)
= \frac{1}{2}\,\frac{g_i^2}{\lambda_i^{2}}.
\end{align*}

\paragraph{Boundary case.} In the boundary case ($\omega > 0$), the maximizer $u^\star$ solves the KKT system for
\begin{align*}
\max_{u:\,\|u\|=\sqrt{d}}\ \rho\,g^\top u+\frac{\rho^2}{2}\,u^\top \Lambda u.
\end{align*}
The stationarity states $u = (2\omega I - \rho^2 \Lambda)^{-1} \rho g$, which is well-defined when $2\omega I - \rho^2\Lambda \succ 0$, i.e.,
$\omega > \frac{\rho^2}{2} \lambda_{\text{max}}$. The correct $\omega$ is chosen so that $\norm{u(\omega)} = \sqrt{d}$ holds. Note that $\norm{u(\omega)}$ is strictly decreasing in $\omega \in (\frac{\rho^2}{2} \lambda_{\text{max}}, \infty)$ since
\begin{align*}
    \norm{u(\omega)}^2 = \sum_{i=1}^d \frac{\rho^2 g_i^2}{(2\omega-\rho^2\lambda_i)^2}
\end{align*}
is strictly decreasing from $\infty$ (assume that $g_1\neq0$) to $0$. Therefore,  there is a unique solution $\omega^{\star} > \max(\frac{\rho^2}{2} \lambda_{\text{max}}, 0)$. Then we can compute $u^{\star} = (2\omega^{\star} I - \rho^2 \Lambda)^{-1} \rho g$ since
\begin{align*}
    \rho g + \rho^2\Lambda u^{\star} - 2\omega^{\star} u^{\star} = 0. \tag{S}
\end{align*}
Next, we compute the regularizer sensitivity $\psi_i$. Since $\omega^{\star}$ and $u^\star$ are functions of $\lambda_i$, we differentiate both sides of (S) with respect to $\lambda_i$ and obtain
\begin{align*}
\rho^2 E_i u^\star+\rho^2\Lambda\,\frac{du^\star}{d\lambda_i}-2\frac{d\omega^{\star}}{d\lambda_i}u^\star-2\omega^{\star}\frac{du^\star}{d\lambda_i}=0,
\end{align*}
where $E_i$ is a diagonal matrix with a 1 at entry $i$ and 0's elsewhere. Differentiating both sides of the constraint that $(u^{\star})^\top u^{\star}=d$ gives
\begin{align*}
2(u^\star)^\top \frac{du^\star}{d\lambda_i}=0. \tag{C}
\end{align*}
Therefore, differentiating $\varphi(u^{\star})$ leads to
\begin{align*}
\frac{d}{d\lambda_i}\varphi(u^\star)
= \rho\,g^\top\frac{du^\star}{d\lambda_i}
+\frac{\rho^2}{2}\Big((u^\star)^\top E_i u^\star
+ 2(u^\star)^\top\Lambda\frac{du^\star}{d\lambda_i}\Big).
\end{align*}
Using stationarity (S) to replace $\rho g$ by $2\omega^{\star} u^\star-\rho^2\Lambda u^{\star}$ gives
\begin{align*}
\rho\,g^\top\frac{du^{\star}}{d\lambda_i} =(2\omega^{\star} u^{\star}-\rho^2\Lambda u^\star)^\top\frac{du^\star}{d\lambda_i} =2\omega^{\star} (u^{\star})^\top\frac{du^\star}{d\lambda_i}-\rho^2 (u^\star)^\top\Lambda\frac{du^\star}{d\lambda_i}.
\end{align*}
By (C), $(u^\star)^\top du^\star/d\lambda_i=0$, so the first term vanishes. Therefore, 
\begin{align*}
\frac{d}{d\lambda_i}\varphi(u^\star) &= -\rho^2 (u^\star)^\top\Lambda\frac{du^\star}{d\lambda_i}
+\frac{\rho^2}{2}(u^\star)^\top E_i u^\star
+\rho^2 (u^\star)^\top\Lambda\frac{du^\star}{d\lambda_i} \\
& = \frac{\rho^2}{2}\,(u_i^\star)^2 \\
& = \frac{\rho^4 g_i^2}{2(2\omega^{\star}-\rho^2\lambda_i)^2} \numberthis \label{dlambda_i}.
\end{align*}
Therefore, the regularizer sensitivity $\phi_i(\omega^{\star}, \lambda_i)$ for an arbitrary $\lambda_i$ is
\begin{align}
    \phi_i(\omega^{\star}, \lambda_i) = \frac{d}{d\lambda_i}\varphi(u^\star) = \frac{\rho^4 g_i^2}{2D_i^2}.
\end{align}
where $D_j \coloneqq 2\omega^{\star} - \rho^2\lambda_j > 0$. Differentiating the sensitivity w.r.t. $\lambda_i$ informs us whether the sensitivity is constant across all eigenvalues as in the average-case SAM regularizer, or increases as in the worst-case SAM and $t$-SAM regularizer $R_{\infty}$. To proceed, we track how $\omega^{\star}$ shifts when $\lambda_i$ changes by implicitly differentiating the secular equation
\begin{align}
    \psi(\omega^{\star}, \{\lambda_j\}) \coloneqq \sum_{j=1}^d \frac{\rho^2g_j^2}{(2\omega^{\star} - \rho^2\lambda_j)^2} = d \label{eq:secular}.
\end{align}
Treat $\psi$ as a function of two variables, including $\omega^{\star}(\lambda_i)$, the function value with parameter $\lambda_i$, and the parameter $\lambda_i$ itself. Differentiating both sides of Eq.~(\ref{eq:secular}) w.r.t. $\lambda_i$ leads to
\begin{align}
    \frac{\partial \psi}{\partial \omega^{\star}} \frac{\partial \omega^{\star}}{\partial \lambda_i} + \frac{\partial \psi}{\partial \lambda_i} = 0 \quad \Longrightarrow \quad \frac{\partial \omega^{\star}}{\partial \lambda_i} = - \frac{\partial \psi / \partial \lambda_i}{\partial \psi / \partial \omega^{\star}} \label{eq:chain_rule}
\end{align}
by the chain rule. Further, we differentiate $\psi$ w.r.t. $\omega^{\star}$ and $\lambda_i$ respectively and have
$$ \frac{\partial \psi}{\partial \omega^{\star}} = -4 \sum_{j=1}^d \frac{\rho^2g_j^2}{D_j^3},\qquad \frac{\partial \psi}{\partial \lambda_i} = \frac{2\rho^4g_i^2}{D_i^3}.$$
Therefore, we apply the above to Eq.~(\ref{eq:chain_rule}) and obtain
\begin{align}
    \frac{d\omega^{\star}}{d\lambda_i} = \frac{\frac{\rho^2g_i^2}{D_i^3}}{2 \sum_{j=1}^d \frac{g_j^2}{D_j^3}}.
\end{align}
Now we can compute 
\begin{align}
    \frac{\partial \phi}{\partial \lambda_i} = -\frac{\rho^4g_i^2}{D_i^3} \frac{\partial D_i}{\partial \lambda_i} = -\frac{\rho^4g_i^2}{D_i^3} \left[2\frac{\partial \omega^{\star}}{\partial \lambda_i} - \rho^2\right] = \frac{\rho^6g_i^2}{D_i^3} \left[1-\frac{\frac{g_i^2}{D_i^3}}{\sum_{j=1}^d \frac{g_j^2}{D_j^3}} \right] \geq 0,
\end{align}
which indicates that the sensitivity of $R_{\infty}$ to any arbitrary $\lambda_i$ grows when $\lambda_i$ increases. This is in contrast with the average-loss SAM, where the influence of all the eigenvalues is always equal. 

\subsection{Choosing $t$} \label{app:proof:choose_t} 

Denote the random variable $X = \rho g^{\top} u + \frac{\rho^2}{2}u^{\top}\Lambda u, \;u\sim\mathcal{U}(\sqrt{d}\mathbb{B}^d)$ and note that $m \leq X \leq M$ with $m=\frac{\rho^2 d}{2}\min(\lambda_{\text{min}}, 0) - \rho \sqrt{d}\norm{g}$ and $M=\frac{\rho^2 d}{2}\max(\lambda_{\max}, 0) + \rho \sqrt{d}\norm{g}$. By Hoeffding’s lemma, for any $t \in \mathbb{R}$, $\mathbb{E}\left[e^{t X}\right]\leq \exp\Big(t\mathbb{E}[X]+{\frac{t^2(M-m)^2}{8}}\Big)$. Then by Jensen's inequality,
$$R_t = \frac{1}{t}\log(\mathbb{E}\left[e^{t X}\right]) \leq \mathbb{E}[X]+{\frac{t(M-m)^2}{8}}.$$
Therefore, to keep $R_t$ within $\varepsilon$ from the expectation $\mathbb{E}[X]=\frac{\rho^2d}{2(d+2)} \sum_{i=1}^d \lambda_i$, which is the sharpness regularizer in the average-case SAM objective, it suffices to take 
\begin{align}
    t \leq \frac{8\varepsilon}{(M-m)^2} \leq \frac{32\varepsilon}{\rho^2d\,[\rho\sqrt{d}\max(|\lambda_{\text{max}}|,|\lambda_{\text{min}}|) + 4\norm{g}]^2}.\label{eq:choose_t}
\end{align}
Since $\rho$ is usually chosen as $\rho\leq \sqrt{d}$ in zeroth-order optimization, the effective range of $t$ is $d$-independent. The remaining parameters, such as $\lambda_{\text{max}}$, are problem-dependent, similar to the generalization bounds presented in prior literature \citep{tsam,aminian2025generalization}. Therefore, in practice, one needs to find the $t$ that yields the best validation performance on the task of interest. Through our experiments with RoBERTa-Base under $t=\{0, 1, 5, 20\}$, however, we observe that $t=1$ is a safe choice for a preliminary trial since it almost always yields superior performance to $t=0$ (Figure~\ref{fig:t_sensitivity}): We find that $t=1$ consistently matches or outperforms MeZO.

\subsection{Low-Dimensional Examples} \label{app:toy}
\paragraph{Linear regime.} We generate the piecewise-linear $f$ by discretizing the function value surface of $h(x,y)=0.07(8x^2+10y^2)+0.14$. By forming $q$ triangles (i.e., planes) $\{P_j\}_{j\in[q]}$ that intersect with $h(x,y)$, we obtain $f(x,y)\coloneqq\min_j P_j(x,y)$ that is piecewise-linear as desired. For the experiments, we run both zeroth-order methods with $k=500$ and $\rho=0.5$ for 40 iterations. We use $t=1$.

\paragraph{Stationary regime.} We define $f: \mathbb{R}^2 \to \mathbb{R}$ by $f(x,y) = \frac{1}{5} [(x^2 - 1)^2 + \frac{1}{2} x (x^2 - 1)^2 + (1 + 2(1 - x)) y^2]$. In the experiments, we consistently start from the initialization at $(0,1)$. We run gradient descent for 50 iterations and zeroth-order methods with $k=500$ and $\rho=0.8$ for 100 iterations. We use $t=1$.
\section{Experiments} \label{app:exp}

In this section, we present our experiment setup and additional results, including sharpness measurements and results under different values of $t$.

\subsection{Experiment Setup} \label{app:exp:setup}

Our code for RoBERTa and OPT experiments is adopted from \citet{mezo} and we use the same data processing workflow and prompt templates as theirs. 

For all zeroth-order methods, we follow prior work \citep{dpzero,mezo} and sample the perturbations in zeroth-order methods from $\mathcal{N}(0,I_d)$ due to the concentration of measure in high-dimensions and the empirical observations that sampling from $\mathcal{N}$ and $\mathcal{S}$ yield very similar performance \citep{mezo,dpzero}. We set $\rho=0.002$ for both RoBERTa and OPT and use $k=5$.

For all SAM variants, we use $\mathcal{U}(r\mathbb{S}^{d-1})$ as the perturbation distribution with $r$ tuned from $\{0.003,0.005,0.01,0.03,0.05\}$, consistent with prior first-order SAM papers \citep{tsam}. We tune $t$ from $\{1, 5, 20\}$ and select the best one based on validation performance. We use $k=5$ for all SAM and TSAM experiments except for using $k=3$ for TSAM on the SQuAD-OPT experiment due to memory constraints.

\paragraph{RoBERTa-Base experiments.} All the first-order methods run for a maximum of 200 epochs, and all the zeroth-order ones run for a maximum of 700 epochs, with early stopping enabled. Note that though zeroth-order methods run for a larger number of iterations, each iteration is much faster and more memory-efficient than the first-order counterparts (see comparison in \citet{mezo}). For SGD, SAM, ESAM, and TSAM, we tune the batch-size from $\{8,32\}$ and $\eta\in\{2\text{e}-3, 1\text{e}-3, 5\text{e}-4, 2\text{e}-4, 1\text{e}-4\}$. For MeZO and ZEST, we fix the batch-size to 128 and tune $\eta\in\{2\text{e}-5,1\text{e}-5,5\text{e}-6\}$.

\paragraph{OPT-1.3B experiments.} We run the first-order methods for maximally 30 epochs (or 3750 steps), and we run the zeroth-order ones for maximally 20K steps. Following the baseline \citep{mezo}, we fix the batch-size to be 8 for first-order methods and 16 for zeroth-order ones. For SGD, we tune $\eta\in\{5\text{e}-5, 1\text{e}-5, 5\text{e}-6\}$; for SAM, ESAM, and TSAM, we tune $\eta\in\{5\text{e}-2, 1\text{e}-2, 1\text{e}-3\}$; for MeZO and ZEST, we tune $\eta\in\{5\text{e}-6, 2\text{e}-6, 1\text{e}-6\}$.


\subsection{Experiment Results} \label{app:exp:result}

\begin{figure*}[h!]
    \centering
    \includegraphics[width=0.85\textwidth]{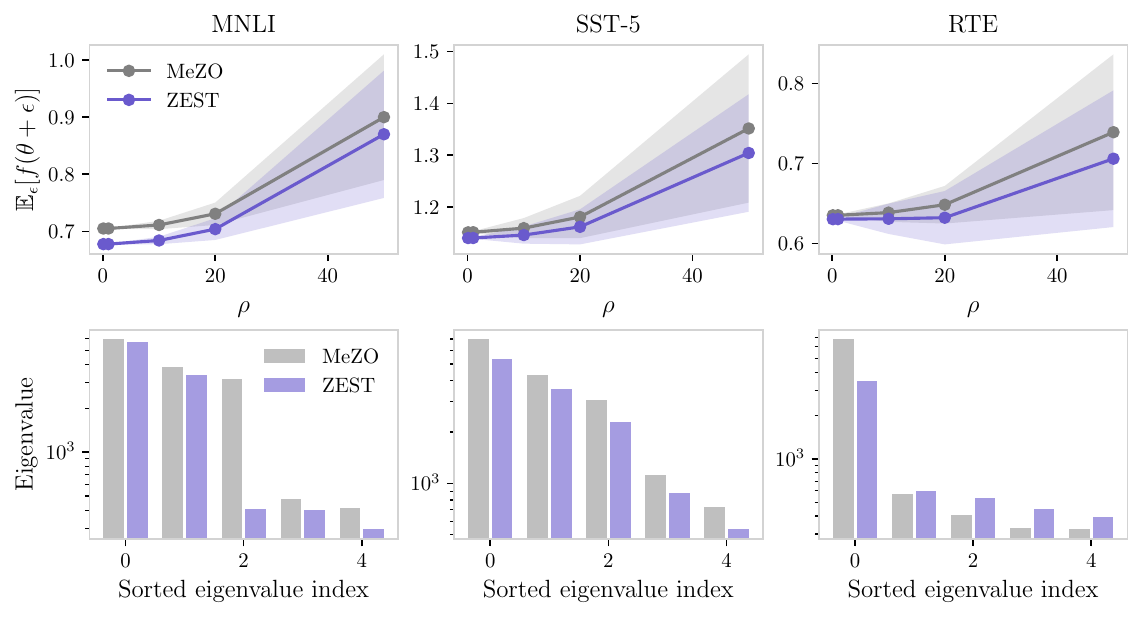}
    \caption{Sharpness of the solutions found by MeZO and ZEST on MNLI, SST-5, and RTE with RoBERTa-Base. Upper: Sharpness measured by $\mathbb{E}_{\norm{\epsilon}\leq \rho}[f(x+\epsilon)]$. The scatters denote the average loss among 500 sampled perturbations and the shade denotes the standard deviation. Lower: Sharpness measured by the top-5 eigenvalues of $\nabla^2 f(x)$. The results suggest that ZEST yields flatter solutions in terms of both the robustness to parameter perturbations and largest curvature of the loss landscape at the arrived minimum, which agrees with our theoretical analysis in Section~\ref{sec:sharpness}.}
    \label{fig:flatness_more}
\end{figure*}

\end{document}